\documentclass[runningheads]{llncs}

\usepackage{eccv}

\usepackage{eccvabbrv}

\usepackage{graphicx}
\usepackage{booktabs}

\usepackage{amssymb}
\usepackage{pifont}  

\usepackage{multirow}
\usepackage{subcaption} 

\usepackage[accsupp]{axessibility}  

\usepackage{hyperref}

\usepackage{orcidlink}

\begin{document}

\title{OmniDiT: Extending Diffusion Transformer to Omni-VTON Framework} 

\titlerunning{OmniDiT}

\author{Weixuan Zeng\inst{1\dag} \and
Pengcheng Wei\inst{2\dag}  \and
Huaiqing Wang\inst{3} \and
Boheng Zhang\inst{3} \and
Jia Sun\inst{3}\and
Dewen Fan\inst{3}\and
Lin HE\inst{3}\and
Long Chen\inst{3}\and
Qianqian Gan\inst{3}\and
Fan Yang\inst{3}\and
Tingting Gao\inst{3}}

\authorrunning{~Zeng et al.}

\institute{The Chinese University of Hong Kong, ShenZhen \and
Beihang University \and
KuaiShou \\
\email{weixuanzeng@link.cuhk.edu.cn\text{,}  22373151@buaa.edu.cn}\\
\email{\texttt{\{wanghuaiqing,yangxiao16,sunjia05,fandewen, \\
helin05,chenlong11,ganqianqian,yangfan,lisize\}@kuaishou.com}}\\
}

\maketitle

\let\thefootnote\relax\footnotetext{$^{\dag}$Work done during an internship at KuaiShou Technology.}

\begin{figure}[h]
  \centering
  \includegraphics[
      height=6.2cm, 
      page=7, 
      trim=0cm 1cm 0cm 1cm, 
      clip               
  ]{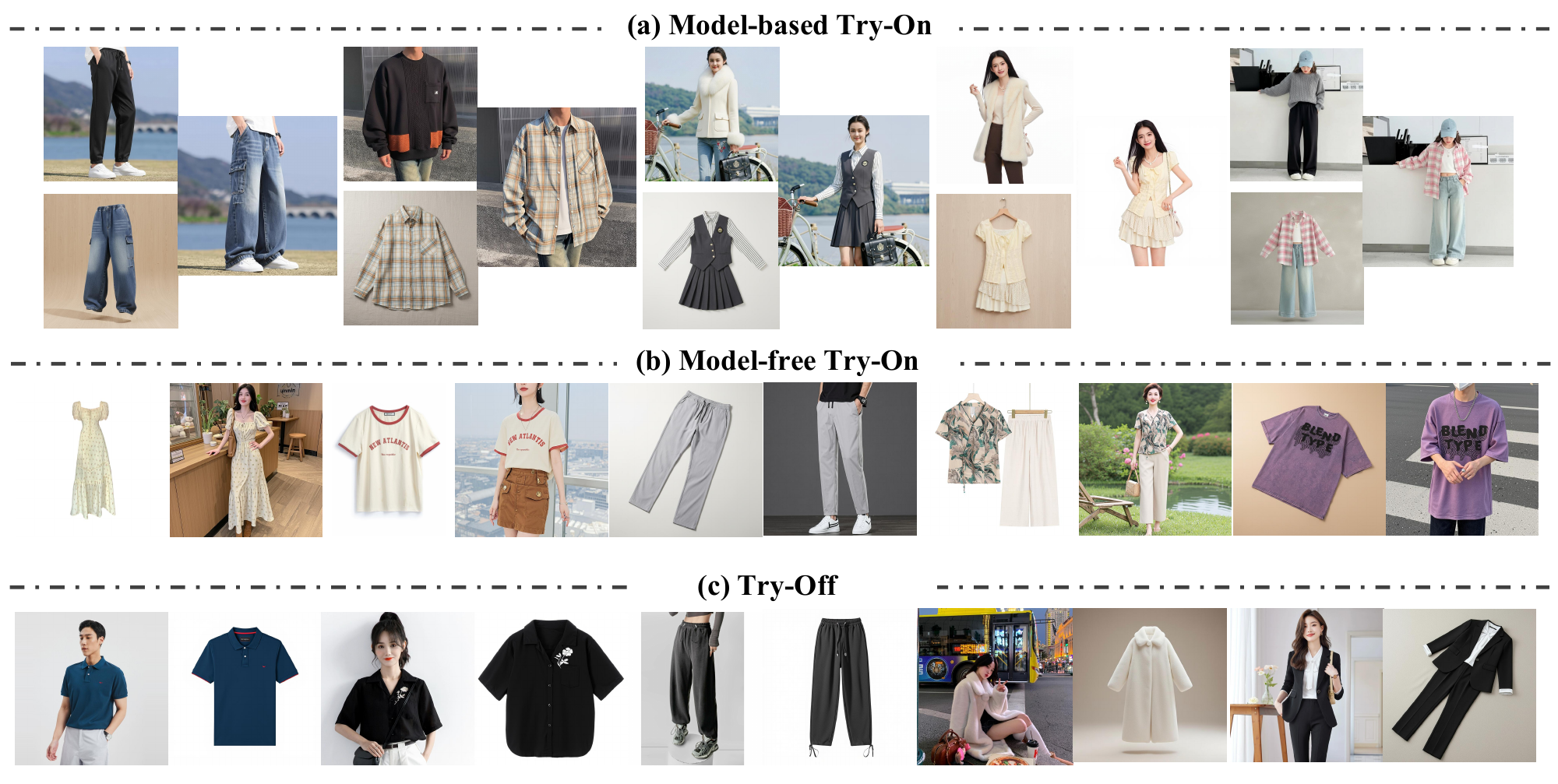}
  \caption{Generated images by our OmniDiT trained on our dataset Omni-TryOn. Our unified model supports above three tasks and keeps a strong consistency. Please zoom in to see details preservation.
  }
  \label{fig:showcase}
\end{figure}

\begin{abstract}
Despite the rapid advancement of Virtual Try-On (VTON) and Try-Off (VTOFF) technologies, existing VTON methods face challenges with fine-grained detail preservation, generalization to complex scenes, complicated pipeline, and efficient inference. To tackle these problems, we propose OmniDiT, an omni Virtual Try-On framework based on the Diffusion Transformer, which combines try-on and try-off tasks into one unified model. Specifically, we first establish a self-evolving data curation pipeline to continuously produce data, and construct a large VTON dataset Omni-TryOn, which contains over 380k diverse and high-quality garment-model-tryon image pairs and detailed text prompts. Then, we employ the token concatenation and design an adaptive position encoding to effectively incorporate multiple reference conditions. To relieve the bottleneck of long sequence computation, we are the first to introduce Shifted Window Attention into the diffusion model, thus achieving a linear complexity. To remedy the performance degradation caused by local window attention, we utilize multiple timestep prediction and an alignment loss to improve generation fidelity. Experiments reveal that, under various complex scenes, our method achieves the best performance in both the model-free VTON and VTOFF tasks and a performance comparable to current SOTA methods in the model-based VTON task.

\end{abstract}

\section{Introduction}
\label{sec:intro}
With the breakthroughs of Diffusion Models (DMs), Virtual Try-On (VTON) has attracted considerable focus due to its promising market prospect in e-commerce. VTON which aims to superimpose given garments onto specific model images \cite{wang2024stablegarment,chen2024anydoor,gou2023taming,kim2024stableviton,ECCV2022,xie2023gp}, provides an immersive and personalized shopping experience \cite{wang2025jco}.  Early approaches relied on generative adversarial network (GAN) \cite{goodfellow2014generative}, which typically added a warping module to achieve the semantic correspondence between the garments and the human body, and a generator module to fit the warped garments onto the body \cite{choi2021viton, ge2021disentangled, ge2021parser, men2020controllable, xie2023gp, he2022style, han2018viton}. However, this two-stage process often results in unnatural fits and cannot generalize to complex human poses due to the limited warping process \cite{choi2024improving, xu2025ootdiffusion}. 

Recently, based on latent diffusion models(LDM) \cite{ho2020denoisingdiffusionprobabilisticmodels, rombach2022highresolutionimagesynthesislatent}, many works \cite{zhu2023tryondiffusion, kim2024stableviton, xu2025ootdiffusion, morelli2023ladi, choi2024improving, wang2024stablegarment, sun2024outfitanyone, chong2024catvton, shen2025imagdressing, chen2024magic, lin2025dreamfit, chong2025fastfit} utilized the rich generative prior of pretrained text-to-image (T2I) models and achieved more natural tryon results. U-Net-based LDMs have effectively improved the realism of outfitted images and could generalize to more complicated scenes \cite{zhang2024boowvtonboostinginthewildvirtual, yang2020towards, li2025dit}.

LDM methods fall into two paradigms: mask-based and mask-free approaches. For mask-based ones, a binary human-agnostic mask is extracted and diffusion models are applied to inpaint the garment into the masked area \cite{morelli2023ladi, xu2025ootdiffusion, choi2024improving, chen2024anydoor}. This method greatly depends on the quality of masking, because it will easily generate patchy clothing and visible artifacts if the mask-extractor cannot perfectly segment the target area \cite{jiang2024fitdit,atef2025efficientviton,choi2024improving}. In addition, the mask area easily leads to information leakage, in which the mask shape can tell the model how and where to inpaint the given garment, reducing the difficulty of learning and weakening the generalization ability of the model. Taking these drawbacks into account, some works removed the masking step and applied an end-to-end pipeline, in which users input the garment and the human model, then the model outputs the tryon result \cite{zhang2024boowvtonboostinginthewildvirtual,niu2024pfdmparserfreevirtualtryon,chang2025pemf}. Mask-free diffusion models have emerged as the dominant paradigm
for high-fidelity virtual try-on \cite{wang2025jco}. 

However, the mask-free method also suffers from severe challenges. First, the end-to-end pipeline requires a fully consistent and matching image triple pair, which greatly complicates the construction of training datasets. The triple images demand that the target garments in the clothing image and the tryon image are exactly the same, and that the model image and the tryon image are completely identical in terms of both the model and the background except their target garments. Actually, popular datasets (\eg VITON \cite{han2018viton} and DressCode \cite{morelli2022dress}) still miss the images of matching models. Second, the model must infer all aspects of the garments' presentation, which increases the difficulty of precise and local refinement along the boundaries of clothing or background coherence \cite{wang2025jco}.

Recently, some works have begun exploring the Virtual Try-Off (VTOFF) application: extracting standardized garment images from clothed individuals, which apparently belongs to the inverse task of VTON \cite{velioglu2024tryoffdiff, xarchakos2024tryoffanyone, velioglu2025mgt, guo2025any2anytryon}. From another perspective, VTON systems demand massive high-quality garment lay-flat images, which are laborious to collect from the Internet. So, as a preliminary process, setting up a VTOFF pipeline to produce high-fidelity garment images can benefit VTON systems' sustainable operation. In addition, an ideal VTON model can not only perform model-based try-on (with two reference images), but command the capability of model-free try-on (with only a single garment image as input) \cite{wang2024stablegarment, chen2024magic, lin2025dreamfit, shen2025imagdressing}. 

Currently, most works focus on a single VTON or VTOFF task, which complicates the whole workflow and fails to satisfy users' diverse demands, only Any2anyTryon \cite{guo2025any2anytryon} attempted to combine the three abilities into one unified model. But their method suffers from simple and small training datasets, and cannot generalize to complex scenes.

To tackle above limitations, we propose OmniDiT, an omni VTON and VTOFF framework built upon a Diffusion Transformer (DiT) \cite{peebles2023scalablediffusionmodelstransformers} and mask-free paradigm. Specifically, we employ token concatenation to integrate reference signals and design an adaptive position encoding to distinguish different token blocks. Aiming to reduce computational costs of the attention modules \cite{vaswani2017attention}, we are the first to introduce shifted window attention (SWA) \cite{liu2021swin} into the diffusion model. In addition, to compensate for the performance degradation induced by SWA, we revise the flow matching objective by predicting multiple timesteps during training, promoting the current timestep's prediction to focus on global timesteps' prediction and yield more stable trajectories. Lastly, to enhance the clothing fidelity, we add an alignment loss anchored in the local clothing region. To further boost robustness and versatility, we curate a large scale try-on dataset--Omni-TryOn, which contains over 380k high-quality, diverse poses and scenes' garment, human model and tryon image triple pairs. We develop a self-evolving curation pipeline that combines the mature VTON and VTOFF capabilities to continuously produce data.

In summary, the contributions of our work include the following:
\begin{enumerate}
    \item We inject multiple condition signals into the DiT by concatenating tokens and designing an adaptive position encoding, aiming to combine VTON and VTOFF tasks into one unified framework.
    \item We introduce shifted window attention into the diffusion model to reduce computation complexity.
    \item Our multiple timesteps' prediction training strategy can encourage every timestep's prediction to focus on following prediction, reducing global error. And additional alignment loss can enhance the clothing fidelity.
    \item We propose a self-evolving data curation pipeline, which combines the updated model's abilities to continuously produce high-quality data. Based on the pipeline, we curate a large dataset with over 380k diverse samples. 
\end{enumerate}

\section{Related Works}
\label{sec:related}

\textbf{Image-based Virtual Try-On.}
Image-based virtual try-on has been extensively explored over the recent years, emerging as a promising and formidable task. Early studies based on Generative Adversarial Networks (GANs) \cite{goodfellow2014generative} proceed with a two-stage pipeline. A warping model deforms the given garment into a shape that approximately fits the person’s pose. Then, a GAN-based generator fuses the wrapped clothing into the result image generation \cite{men2020controllable, xie2023gp, lee2022high, yang2023occlumix, wang2018toward, han2019clothflow}. However, these approaches are frequently hampered by visual artifacts from inaccurate warping \cite{chong2025fastfit}. Subsequently, the advent of diffusion models revolutionized the field by reframing the task as an end-to-end conditional image generation, removing the error-prone warping step \cite{baldrati2023multimodal, choi2024improving, chong2024catvton, gou2023taming, morelli2023ladi, shen2025imagdressing, xu2025ootdiffusion, zeng2024cat, zhou2025learning, zhu2023tryondiffusion}. The dominant strategy in these methods involves injecting garment features into the denoising process via sophisticated conditioning mechanisms such as parallel encoder branches (i.e., ReferenceNets), ControlNet \cite{zhang2023adding} or IP-adapter \cite{ye2023ip}. With the advent of Diffusion Transformer architecture \cite{peebles2023scalablediffusionmodelstransformers, labs2025flux}, some works have explored more generalized conditioning schemes, among which token concatenation has achieved the best performance \cite{li2025dit, wu2025less,song2025omniconsistency, mou2025dreamo, guo2025any2anytryon, zhang2025easycontrol}. Despite achieving unprecedented high-fidelity, the model's vast size has significantly increased inference latency, especially multiple conditions as input, hindering their applications in real-world scenarios that demand rapid feedback and multi-item outfit composition \cite{chong2025fastfit}. Besides model-based try-on, another line of applications is model-free try-on, which needs to generate the human model and corresponding tryon results. This task is more difficult, because it requires the try-on system to imagine the human model based on pure text prompts and maintain the garment conditions' details. Current works mainly employ the condition-injection paradigm to achieve the goal \cite{shen2025imagdressing, guo2025any2anytryon, lin2025dreamfit, wang2024stablegarment}.

\noindent
\textbf{Virtual Try-Off.}
While most existing works focus on virtual try-on, few studies have explored virtual try-off \cite{zhang2022armani, zhang2023diffcloth, zhang2024garmentaligner}, the inverse task of virtual try-on, aiming to reconstruct a clean garment from a tryon image. Early works, such as TileGAN \cite{zeng2020tilegan}, used a two-stage pipeline: a U-Net-like encoder-decoder for coarse synthesis followed by a pix2pix-based refinement. Recent works explored text-guided garment generation based on Diffusion models \cite{guo2025any2anytryon, velioglu2024tryoffdiff, lee2025voost, velioglu2025mgt}. But rare works have attempted to combine try-on and try-off into a unified framework to learn \cite{lee2025voost, guo2025any2anytryon}. This unified setup not only
supports multitask learning but also mitigates architectural, task-specific, and category-specific inductive biases by exposing the model to broader structural variation.

\section{Method}
\label{sec:method}

\subsection{Preliminary}
\label{sec:prelimi}

\textbf{Virtual Try-On and Try-Off.} We mainly focus on mask-free Virtual Try-On and Try-Off. Given a human model image and a garment image as inputs, the model-based try-on pipeline generates a tryon image without relying on any mask condition, and the model-free try-on operates similarly, only removing the human model image. As the inverse task of Virtual Try-On, Virtual Try-Off generates the garment image based on the tryon image.

\noindent
\textbf{Flow Matching.} Unlike Diffusion denoising methods \cite{ho2020denoising, song2020score}, Flow matching \cite{liu2022flow, lipman2022flow} predicts a smooth, continuous velocity field from noise to data in a direct manner. By conducting the forward process by linearly interpolating between noise and data, it can potentially achieve faster and more efficient image synthesis. At timestep $t$, latent $x_t$ is defined as: $x_t = (1 - t)x_0 + tx_1$, where $x_0$ is the clean image, and $x_1 \in N(0,1)$ is the Gaussian noise. A model is trained to directly regress the target velocity given the noised latent $x_t$, timestep $t$, and conditions $c$ (including the text prompt, reference images.).  The flow matching loss can be summarized as follows:
\begin{equation}
    \mathcal{L}_{\text{FM}}(\theta) = \mathbb{E}_{t \sim \mathcal{U}[0,1]} \left[ \| v_\theta(x_t, t, c) - (x_1 - x_0) \|^2 \right]
\label{eq:ssp}
\end{equation}
where $v_{\theta}$ represents the velocity field predicted by a model parameterized by neural network weights $\theta$.


\subsection{Omni-TryOn Dataset Construction}
\label{sec:dataset}

\begin{figure}[t]
  \centering
  \includegraphics[
      height=6.5cm, 
      page=1, 
      trim=0.8cm 0.5cm 0.5cm 0.8cm, 
      clip               
  ]{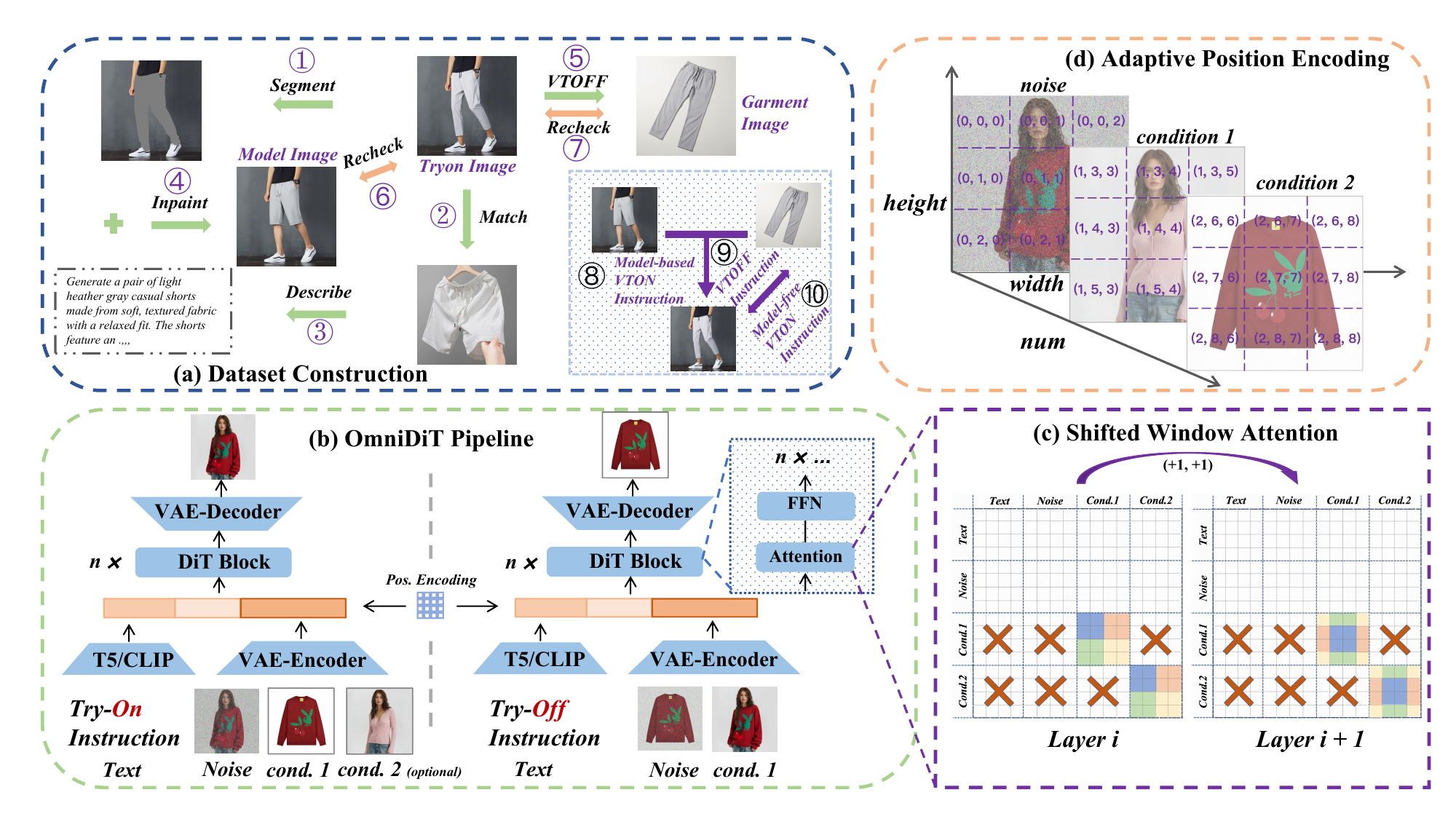}
  \caption{Overview of Our OmniDiT framework. The four blocks demonstrate our (a) dataset construction, (b) OmniDiT model details, (c) shifted windows attention, and(d) adaptive position encoding.
  }
  \label{fig:framework}
\end{figure}

An ideal Virtual Try-On dataset should satisfy three diverse and two consistent requirements: (1) diverse garments: including garment's categories and garment's styles; (2) diverse human models: including models' identities and postures; (3) diverse scenes: including garments and models' scenes; (4) consistent garments: the target garments in garment and tryon image pairs should keep consistent; (5): consistent models: the human model in the model and tryon image pairs should keep consistent. Current public try-on datasets, such as VITON-HD \cite{choi2021viton} and DressCode \cite{morelli2022dress} only contain garment-tryon pairs, missing the matching human model, and cannot meet the standard of diverse scenes and poses, limiting the development of the mask-free VTON techniques. 

\noindent
\textbf{Dataset Construction Pipeline.} 
We first collected two millions of garment display images from the Internet, including 3-5 pure garment images and tryon results for each garment. To construct fully matching and consistent triple pairs, we use powerful VLMs, such as Qwen3-VL \cite{bai2025qwen3} and InternVL-3.5 \cite{wang2025internvl3_5} to filter qualified images, and the criteria include garment or model image classification, watermark detection, content proportion, garments or models' count, garment categories, model type and age, and so on. Among filtered images, we categorize them into two groups based on whether they contain fully matching garments and tryon results: one contains garments and the corresponding tryon results, the other contains only the tryon results without corresponding garments. Then we can employ different workflows to synthesize the expected image pairs: 

\textbf{For garment-tryon pairs}: Due to the available garment and tryon results, we only need to produce a human model which has the same model and scene as the tryon image, but different dressed garment to replace the target garment. As shown in \cref{fig:framework}a, we use a segmentation model \cite{DBLP:journals/corr/abs-2105-15203} to generate a masked human image. Simultaneously, we sample a matching garment based on the target garment's category, the age and gender of the model from our pre-built garment database, which includes several types of garments suitable for diverse models with a wide range of ages and genders. Then employ a VLM to generate a detailed text description of the sampled garment, and utilize an inpainting model to fill in the masked image with the text instruction and obtain the corresponding model image. To ensure consistency among the triple image pairs, employ VLMs to recheck the results. Then, we use the Qwen3-VL model to generate detailed text prompts tailed for three tasks: model-based try-on, model-free try-on and try-off tasks. Finally, to guaranty the aesthetics, we utilize an aesthetic model to score the garment and tryon images for each sample.

\textbf{For only tryon images}: The overall pipeline is similar to that of garment-tryon pairs, only need to attain the corresponding target garment images additionally. Current powerful generation model such as Gemini-2.5-image \cite{google2025gemini} can put off the clothes to perfectly generate the target garment image.

Actually, after collecting the first batch of training data, we can update our original OmniDiT model to acquire the try-on and try-off abilities. Then OmniDiT can generate qualified model images based on tryon results and sampled garments, thus removing the mask and inpainting steps. Also replace the Nano Banana API service to simplify the try-off pipeline. The regenerated triplets are re-filtered via a VLM, yielding an expanded corpus with both high-quality and wider coverage. Overall, our dataset construction pipeline paired with updated OmniDiT model can achieve a self-evolving consequence, continuously producing high-quality data. After multiple iterations, we collected 380k high-quality, diverse samples with 1895 test samples as shown in \cref{tab:dataset}. More details can be found in \cref{app:dataset}.

\subsection{Model Architecture}
\label{sec:model archi}
In this part, we present our method for modifying DiT to adapt virtual try-on and try-off tasks. OmniDiT builds on Flux.1-Kontext-dev \cite{labs2025flux}, which injects the reference image and text prompt along with the noisy signal into its MM-DiT blocks by concatenating all the image and text tokens into a unified sequence. The token concatenation facilitates cross-modal interactions to iteratively guide and refine the synthesis process, thus generating images that faithfully reflect the prompts and reference conditions \cite{wang2025jco}. Several previous works have demonstrated that token concatenation yields the best performance \cite{li2025dit, wu2025less,song2025omniconsistency, mou2025dreamo}. Although the original Flux model only supports one reference image as input, it can accept more reference images by concatenating extra image tokens into the unified sequence, enabling seamless incorporation of various control signals and facilitating high-fidelity, controllable image generation.

\noindent
\textbf{Multi-Condition In-Context Generation.} As shown in \cref{fig:framework}b, the text prompt, the noise and reference images(garment and model images for the model-based try-on, garment images for the model-free try-on, tryon images for the try-off task) are encoded into the same latent space using text and image encoders. Assuming the text tokens $T$, noisy tokens $X$ and reference condition tokens $[C_1,...C_n]$, $n = 2$ or $n = 1$ for our above three tasks, all tokens are concatenated along the sequence dimension $S = [T;X;C_1;...C_n]$ and fed into the model.

\noindent
\textbf{Shifted Window Attention.} Due to the attention's quadratic complexity with respect to the sequence length, adding more high-resolution reference images into the sequence will cause the global self-attention computation unaffordable. To tackle this problem, we are the first to introduce the shifted window attention(SWA) \cite{liu2021swin} into the diffusion model. Considering the characteristic of image generation, we apply SWA only to the reference images. As illustrated in \cref{fig:framework}c, a full reference image is partitioned into several non-overlapping local windows, and the attention is computed within the local windows. To introduce cross-window connections, a shifted window partitioning approach is applied in consecutive attention blocks. The first layer $i$ uses a regular window partitioning strategy in which the 4 × 4 feature map is evenly partitioned into 2 × 2 windows of size 2 × 2 ($M = 2$). Then, the next layer $i + 1$ is shifted from that of the preceding layer, by rolling the windows by $\left( \left\lfloor \frac{M}{2} \right\rfloor, \left\lfloor \frac{M}{2} \right\rfloor \right)$ pixels towards the bottom right direction. The shifted window attention achieves linear complexity by limiting self-attention computation to non-overlapping local windows, and obtains comparable performance with global attention by allowing cross-window connection. Inspired by \cite{zhang2025easycontrol, tan2025ominicontrol}, we also introduce the Causal Conditional Attention design, in which the attention from condition tokens to denoising tokens (noise and text) is blocked, thus further reducing redundant computation and improving efficiency. In the case of two $1024\times$ reference images as input and generate a $1024\times$ image, the inference time is cut down from original 55s to 47s(-14.5\%) after applying SWA on an A800 GPU, and the effect will be more pronounced with higher resolution and more input conditions.

\noindent
\textbf{Adaptive Position Encoding.} After concatenating conditions into the unified sequence, we need reassign the position index of each image token to reduce condition confusion. FLUX.1-Kontext employs a three-dimensional RoPE \cite{su2024roformer} scheme that assigns the position indices (i, w, h) to both text and image tokens. In the original setting, text tokens are assigned a consistent position index of $(0, 0, 0)$, while noisy image and reference tokens are allocated shared position indices $(w, h)$ where $w \in [0, W - 1]$ and $h \in [0, H - 1]$. Only one different point lies in the first dimension index that $i = 0$ for noisy token and $i = 1$ for reference tokens. To inherit the original position priors and adapt to multiple conditions, we align the reference image with the noisy image in the diagonal position, as shown in \cref{fig:framework}d. We remain the position index of the text and noisy image, and the position index for the token $(w, h)$ in the $i_{th}$ reference image is defined as: 
\begin{equation}
    (\hat{i}, \hat{w}, \hat{h}) = (i, w_{noisy} + \sum_{j = 1}^{i - 1}w_{ref_j} + w,  h_{noisy} + \sum_{j = 1}^{i - 1}h_{ref_j} + h)
\end{equation}
where $w_{noisy}$ and $h_{noisy}$ represent the width and height of the noisy latent, and $w_{ref_j}$ and $h_{ref_j}$ represent the width and height of the $j_{th}$ reference condition latent. This position index assignment can avoid index overlap and encourage the model to distinguish different conditions.

Actually, the size of condition images is not necessarily same as the noisy image. So we need to interpolate position encodings to ensure spatial alignment inspired by UNO \cite{wu2025less}, and the final position index is represented as: 
\begin{equation}
    (\hat{i}, \hat{w}, \hat{h}) = (i, w_{noisy} + \sum_{j = 1}^{i - 1}w_{ref_j} + w * S_w,  h_{noisy} + \sum_{j = 1}^{i - 1}h_{ref_j} + h * S_h)
\end{equation}
where $S_w = \frac{w_{noisy}}{w_{ref_i}}, S_h = \frac{h_{noisy}}{h_{ref_i}}$ are scaling factors in width and height directions.

\subsection{Training Strategy.}
\textbf{Multiple Timesteps Prediction.}
In practice, most Flow Matching implementations employ a single-step prediction objective, where the model is trained to predict the velocity at isolated random time points $t \sim \mathcal{U}[0,1]$. Although effective, this approach provides no explicit constraint on the temporal consistency of the velocity field across adjacent time steps. Consequently, the learned velocity field may exhibit high-frequency oscillations, leading to unstable numerical integration and degraded generation quality.

In this work, we introduce Multi-Timestep Prediction (MTP), which extends the training objective by unrolling multiple Euler integration steps within a single training iteration and supervising the velocity prediction at each intermediate time point. Single-Timestep Prediction (SSP) trains the model by sampling a single random time $t$ and minimizing the velocity prediction error at that time point, as in \cref{eq:ssp}. Multi-Timestep Prediction (MTP) unrolls $K - 1$ Euler integration steps from time $t$ to $t - (K - 1)\Delta t$, supervising the velocity prediction at each intermediate step:
\begin{equation}
    \mathcal{L}_{\text{MTP}} = \frac{1}{K} \sum_{k=0}^{K-1} \mathbb{E} \left[ \| v_\theta(x_{t_k}, t_k, c) - (x_1 - x_0) \|^2 \right]
\end{equation}
where $x_{t_{k+1}} = x_{t_k} + (t_{k+1} - t_k) \cdot v_\theta(x_{t_k}, t_k)$ and $t_k = t - k\Delta t$.
We prove that MTP implicitly imposes a temporal smoothness constraint on the velocity field, effectively reducing its Lipschitz constant and yielding more stable trajectories. Detailed theoretical analysis is demonstrated in \cref{app:analysis}.

\noindent
\textbf{Alignment Loss.}
We additionally employ an alignment loss to enhance the fidelity of garment regions \cite{xu2025withanyone}. Denoting a feature extraction model as $\mathcal{E}$ (Dinov2 \cite{oquab2023dinov2} in our work), the generated image as $G$ and the ground-truth(GT) image as $GT$, we segment the ground-truth's garment region $M$ in advance, and extract the mask region's feature to compute the cosine distance between GT-aligned garment features of the generated and ground-truth images as:
\begin{equation}
    \mathcal{L}_{\text{align}} = 1 - \cos(\mathcal{E}(M \odot GT), \mathcal{E}(M \odot G))
\end{equation}

\noindent
The overall training objective is a weighted sum of the above two losses:
\begin{equation}
    \mathcal{L} = \mathcal{L}_{\text{MTP}} + \lambda \mathcal{L}_{\text{align}}
\end{equation}
where $\lambda = 0.10$ aims to control the contribution of alignment loss.

\begin{table}[t]
\centering
\caption{Quantitative comparison on \textbf{VITON-HD} and \textbf{DressCode} dataset for the \textbf{model-based try-on} task. We multiply KID by 1000 for better comparison. The best and the second best results are denoted as \textbf{Bold} and \underline{underline}, respectively.}
\label{tab:tryon_w_m}
\begin{tabular}{l|cccc|cccc}
\toprule
Dataset & \multicolumn{4}{c}{VITON-HD} & \multicolumn{4}{c}{DressCode} \\
\cmidrule(lr){2-5} \cmidrule(lr){6-9}
Method
 & FID$\downarrow$ & KID$\downarrow$ & SSIM$\uparrow$ & LPIPS$\downarrow$ & FID$\downarrow$ & KID$\downarrow$ & SSIM$\uparrow$& LPIPS$\downarrow$ \\
\midrule
IDM-VTON\cite{choi2024improving} &  \underline{6.1150}  &  1.0458  &  0.8654  &  \textbf{0.0744}  &  5.2637  & 2.0087   &  0.8874  &  0.0723 \\
OOTDiffusion\cite{xu2025ootdiffusion} &6.5364  &  \underline{0.9353}  &  0.8403  &  0.0875  &  3.9241  &  0.7381  &  0.8868  &  \underline{0.0678 }\\
StableGarment\cite{wang2024stablegarment} &  9.5081  &  3.1353  &  0.8353  &  0.0925  &  6.2972  &  2.5633  &  0.8762  &  \textbf{0.0641} \\
CatVTON\cite{chong2024catvton} & \textbf{6.1124}  &  0.9448  &  \underline{0.8746}  &  0.0976 &  3.4869  &  0.8294  & \textbf{0.8984}  &  0.0869\\
FastFit\cite{chong2025fastfit} & 7.3009  &  1.3189  &  0.8535  &  0.1208 &  \textbf{3.3824}  &  \underline{0.5449}  &  0.8915  &  0.0696 \\
FitDiT\cite{jiang2024fitdit} & 11.4429  &  4.2931  &  0.7827  &  0.1603 &  5.6900  &  2.1629  &  0.8437  &  0.1105 \\
Any2anyTryon\cite{guo2025any2anytryon} &  14.6365  &  7.1984  &  0.7795  &  0.1888 &  8.0300  &  3.5759  &  0.8305  &  0.1529 \\
Jco-MVTON\cite{wang2025jco} & 11.7850  &  4.8869  &  0.7888  &  0.1587 &  8.7186  &  3.9160  &  0.8462  &  0.1183 \\

\midrule

\textbf{OmniDiT} & 6.4564  &  \textbf{0.7502}  &  \textbf{0.8838}  &  \underline{0.0784}  &  \underline{3.4522}  &  \textbf{0.5245}  &  \underline{0.8980}  &  0.0789 \\
\bottomrule
\end{tabular}
\end{table}

\begin{table}[ht]
\centering
\caption{Quantitative results on \textbf{VITON-HD} dataset for the \textbf{model-free try-on} task. The best and the second best results are denoted as \textbf{Bold} and \underline{underline}.}
\label{tab:tryon_wo_m_vitonhd}
\begin{tabular}{l|ccccccc}
\toprule
Method
 & DINO-I $\uparrow$ &  CLIP-I $\uparrow$ & SSIM $\uparrow$  &  LPIPS $\downarrow$ & FID $\downarrow$ & KID $\downarrow$ & FFA $\uparrow$\\
\midrule
MagicClothing\cite{chen2024magic} & 0.3091  &  0.7014  &  0.5750  &  0.4689  &  46.6366  &  37.0567 & 0.5619 \\
StableGarment\cite{wang2024stablegarment} & \textbf{0.3710}  &  0.7319  &  0.5331  &  0.5056  & 55.9763  &  21.2100 & \textbf{0.5818}   \\
Any2anyTryon\cite{guo2025any2anytryon} & 0.2819  & 0.7880  & 0.6716  & 0.3921  & 73.3430  & 48.2852 & 0.5337 \\
DreamFit\cite{lin2025dreamfit} & 0.3626  &  \underline{0.7935} &  0.6826 &  \underline{0.3039} & \underline{15.6002}  & \underline{5.4227} & 0.5683 \\
IMAGDressing\cite{shen2025imagdressing} & 0.3459  &  0.7845 &  \underline{0.6915} & 0.3352  & 17.6286  & 8.7806 & 0.5640   \\

\midrule

\textbf{OmniDiT} & \underline{0.3638}  &  \textbf{0.8022}  &  \textbf{0.7014}  &  \textbf{0.2844}  &  \textbf{12.4448}  &  \textbf{3.2455} & \underline{0.5712} \\
\bottomrule
\end{tabular}
\end{table}

\begin{table}[h]
\centering
\caption{Quantitative comparison on \textbf{VITON-HD} dataset for the \textbf{try-off} task. The best and the second best results are denoted as \textbf{Bold} and \underline{underline}, respectively.}
\label{tab:tryoff_vitonhd}
\begin{tabular}{l|ccccccc}
\toprule
Method
 & CLIP-I $\uparrow$ & ms-SSIM $\uparrow$ &  LPIPS $\downarrow$ & FID $\downarrow$ & clip-FID $\downarrow$ & KID $\downarrow$  & DISTS $\downarrow$ \\
\midrule
TryOffDiff\cite{velioglu2024tryoffdiff} & 0.9177    &  \textbf{0.6991}  &  \underline{0.1933}  &  17.1050  &  7.0578 & 5.2198 & 0.2297   \\
MGT\cite{velioglu2025mgt} & 0.9015   &  0.6443  &  0.2508  &  21.8485  &  8.0539 & 8.8200 & 0.2716   \\
TryOffAnyone\cite{xarchakos2024tryoffanyone} & \underline{0.9242} &  0.6419  &  0.2109  &  \underline{11.3476}  &  \underline{5.1583} & \underline{1.9186} & \underline{0.2169}   \\
Any2anyTryon\cite{guo2025any2anytryon} & 0.8749    & 0.5607  & 0.3126  & 20.4215  & 8.0235 & 5.4847 &  0.2972 \\

\midrule

\textbf{OmniDiT} & \textbf{0.9391}  &  \underline{0.6782}  &  \textbf{0.1859}  &  \textbf{10.5479}  &  \textbf{2.4875}  &  \textbf{1.5611}  &  \textbf{0.1958} \\
\bottomrule
\end{tabular}
\end{table}

\section{Experiments}
\label{sec:experiments}

\subsection{Experimental Setup}
\label{sec:Setup}
\textbf{Training Data.} To make a fair comparison, we train a unified model on two publicly available datasets: VITON-HD \cite{choi2021viton}, containing 13,679 high-resolution pairs of frontal half-body models and upper-body garments, and DressCode \cite{morelli2022dress} with 53,792 pairs of full-body models and upper-body, lower-body, and dress garments. We use the official train/test splits provided by both datasets. All reference images are resized to $512 \times 768$ as input and generate $768 \times 1024$ images. In addition, we train a unified model on our Omni-TryOn with a resolution of $768\times768$ and reference images of $512 \times 512$ as supplementary.

\noindent
\textbf{Implementation Details.} All experiments are conducted using LoRA \cite{hu2022lora} with a rank of 128 on $8 \times$ NVIDIA H200 GPUs. For SWA, the window size $M$ is $16$, and for MTP, we predict $2$ timesteps for every optimization step, namely $K = 2$.
We adopt a two-stage training approach: the first stage optimizes only the model-free try-on and try-off tasks, and the second stage optimizes all three tasks for better consistency. 
In inference, we set the denoising steps to 30 and the guidance scale to 4. Additional details can be found in \cref{app:setting}.

\subsection{Model-based Try-On}

\begin{figure}[ht]
  \centering
  \includegraphics[
      height=6.5cm, 
      page=14, 
      trim=0.8cm 1.5cm 0.5cm 0.8cm, 
      clip               
  ]{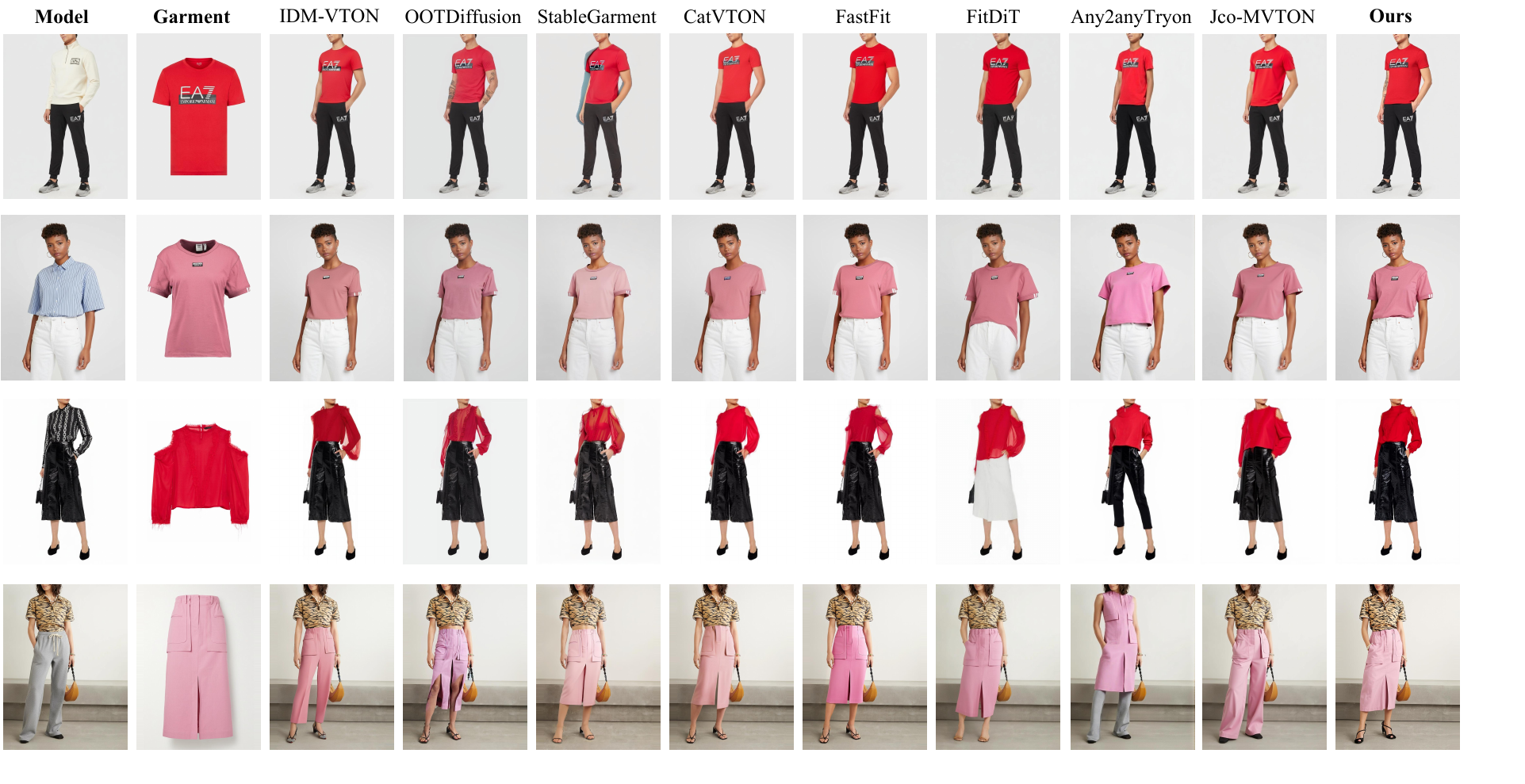}
  \caption{Qualitative comparison of the model-based try-on generation results on the VITON-HD and DressCode benchmarks. Please zoom in to see details preservation.
  }
  \label{fig:tryon_w_m_public}
\end{figure}

We adopt four popular metrics, including Structural Similarity (SSIM) \cite{wang2004image}, Learned Perceptual Image Patch Similarity (LPIPS) \cite{zhang2018unreasonable}, Fréchet Inception Distance (FID) \cite{heusel2017gans} and Kernel Inception Distance (KID) \cite{binkowski2018demystifying}.

For the public benchmarks, the baseline methods contain U-net-based and DiT-based models, including IDM-VTON\cite{choi2024improving}, OOTDiffusion\cite{xu2025ootdiffusion}, StableGarment\cite{wang2024stablegarment}, CatVTON\cite{chong2024catvton}, FastFit\cite{chong2025fastfit}, FitDiT\cite{jiang2024fitdit}, Any2anyTryon\cite{guo2025any2anytryon}, Jco-MVTON\cite{wang2025jco}. As shown in \cref{tab:tryon_w_m}, our model achieves comparable performance in most metrics and exceeds the SOTA methods in KID and SSIM. Compared with the unified model Any2anyTryon, OmniDiT outperforms it in all metrics, confirming our outstanding ability in producing high-quality and consistent tryon images. In \cref{fig:tryon_w_m_public}, we show some qualitative cases on public benchmarks. Our method consistently excels in preserving key attributes: text and logo on the red shirt, color on the pink shirt, texture, and details on the skirts.

In addition, we compare some methods on our Omni-TryOn benchmark. Besides professional try-on methods, we also test general in-context generation methods, including UNO \cite{wu2025less}, DreamOmni2 \cite{xia2025dreamomni2}, DreamO \cite{mou2025dreamo}, OmniGen2 \cite{wu2025omnigen2}. Quantitative results in \cref{tab:tryon_w_m_omni_tryon} demonstrate our method's eminent performance in handling complex scenes. In \cref{fig:tryon_w_m_omni_tryon}, professional try-on methods cannot deal with complex scenes and postures, and other general methods are limited to weak consistency. In contrast, our method fully maintains the garment's details and renders it on the human perfectly, verifying our outstanding performance. More visualization examples can be found in \cref{fig:add_tryon_w_m_1,fig:add_tryon_w_m_2}.

\subsection{Model-free Try-On}

\begin{figure}[ht]
  \centering
  \includegraphics[
      height=5cm, 
      page=15, 
      trim=0.8cm 5.5cm 0.5cm 0.8cm, 
      clip               
  ]{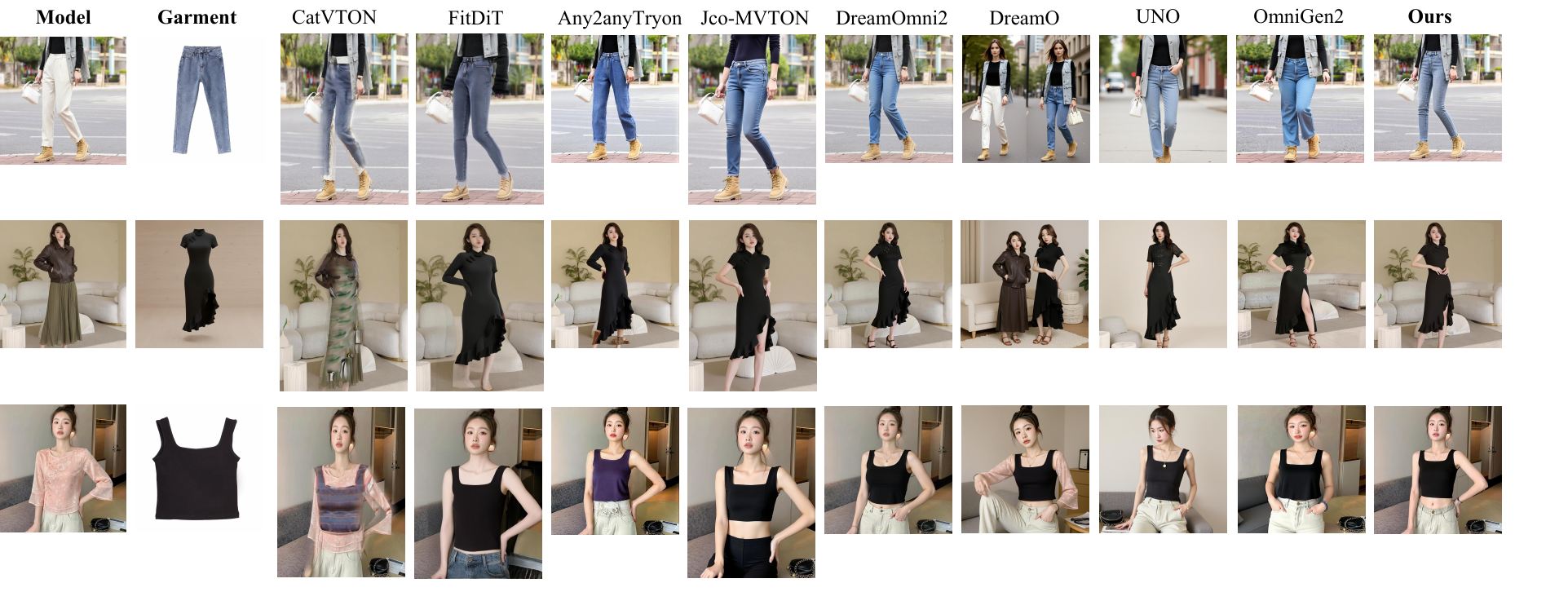}
  \caption{Qualitative comparison of the model-based try-on generation results on our Omni-TryOn benchmark. All methods adopt their optimal resolutions.
  }
  \label{fig:tryon_w_m_omni_tryon}
\end{figure}

\begin{figure}[ht]
  \centering
  \includegraphics[
      height=5.5cm, 
      page=16, 
      trim=0.8cm 2.5cm 0.5cm 0.8cm, 
      clip               
  ]{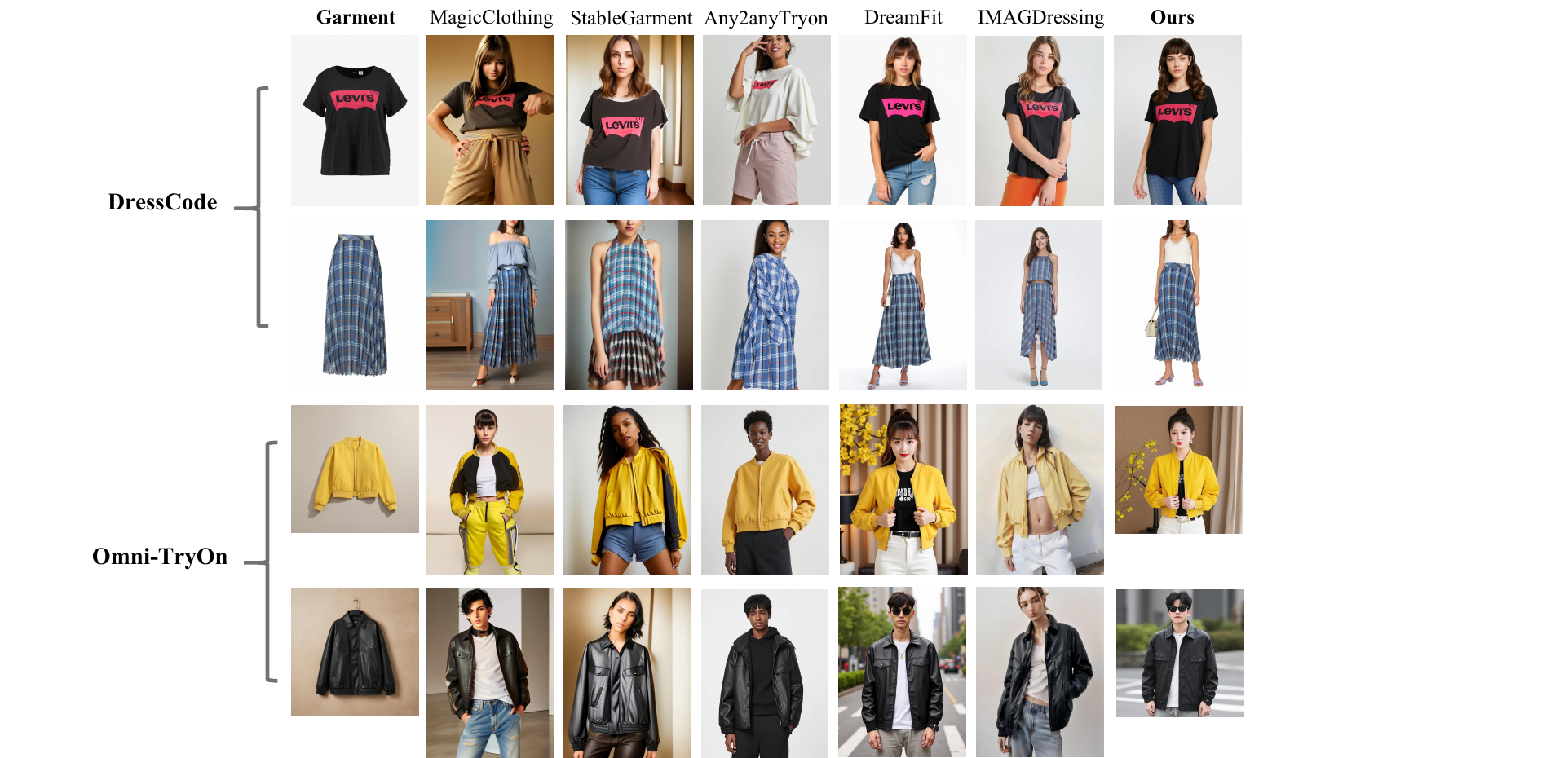}
  \caption{Qualitative comparison of the model-free try-on generation results. All methods adopt their optimal resolution. Please zoom in to see details preservation.
  }
  \label{fig:tryon_wo_m}
\end{figure}

For the model-free try-on task, quantitative metrics include: the visual similarity using DINO \cite{oquab2023dinov2} and CLIP \cite{radford2021learning}, SSIM, LPIPS, FID, KID and FFA \cite{kotar2023these}. Baselines contain MagicClothing \cite{chen2024magic}, StableGarment \cite{wang2024stablegarment}, Any2anyTryon \cite{guo2025any2anytryon}, DreamFit \cite{lin2025dreamfit}, IMAGDressing \cite{shen2025imagdressing}. The model-free try-on task requires the try-on system to imagine a suitable human model based on the given text prompt, and wear the given garment, thus being more challenging than the model-based try-on task. Our method shows a strong capability to create a real human model and maintain better fidelity in texture, color, and details in \cref{fig:tryon_wo_m}. 

Quantitative results in \cref{tab:tryon_wo_m_vitonhd,tab:tryon_wo_m_dresscode,tab:tryon_wo_m_omnitryon} verify that in both 'in shop' and 'in the wild' scenes, our method significantly outperforms all baselines, especially in terms of FID and KID, indicating that our model can capture comprehensive and fine-grained features from reference garments and fit the garment well into the generated human model. More visualization examples can be found in \cref{fig:add_tryon_wo_m}.

\begin{figure}[ht]
  \centering
  \includegraphics[
      height=4.5cm, 
      page=17, 
      trim=0.8cm 6.5cm 6cm 0.8cm, 
      clip               
  ]{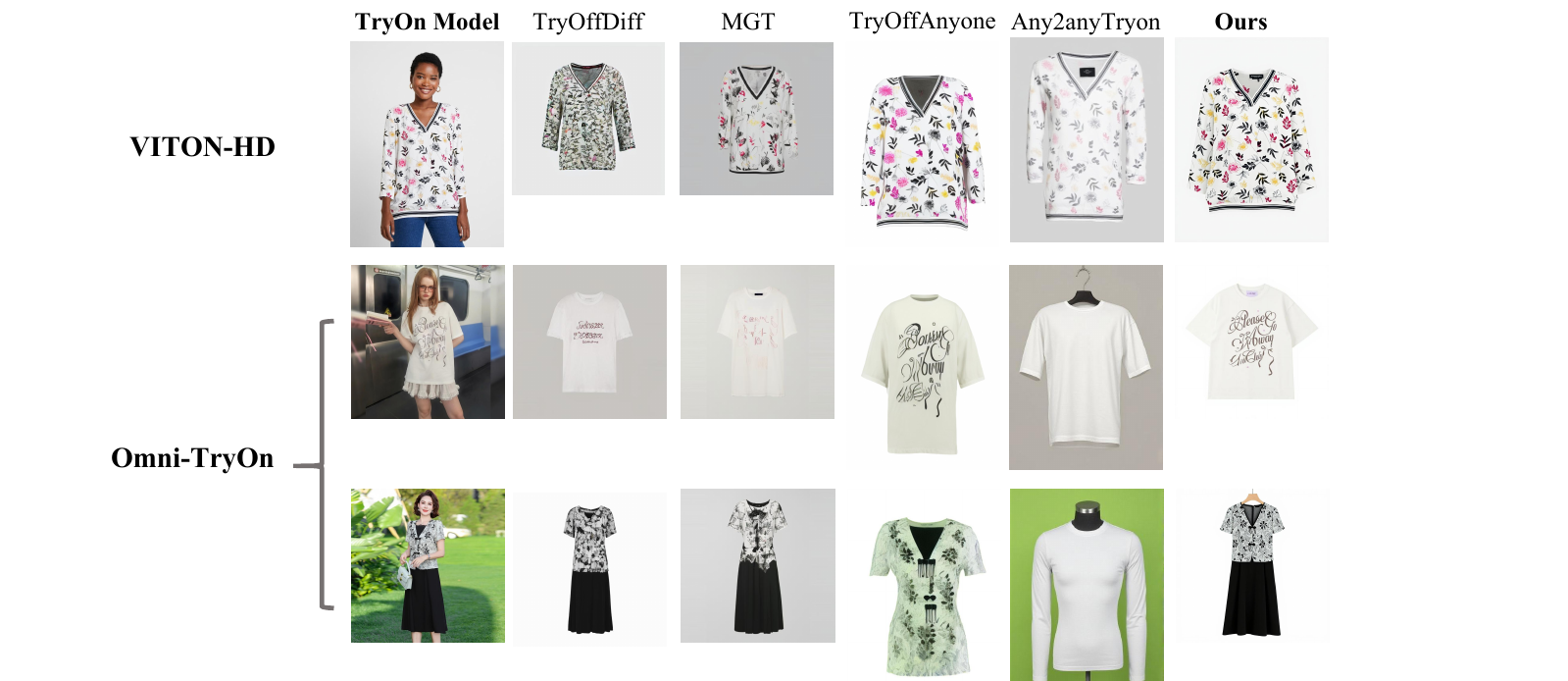}
  \caption{Qualitative comparison of the try-off generation results. All methods adopt their optimal resolution. Please zoom in to see details preservation.
  }
  \label{fig:tryoff}
\end{figure}

\subsection{Try-Off}

For the try-off task, baselines include TryOffDiff \cite{velioglu2024tryoffdiff}, MGT \cite{velioglu2025mgt}, TryOffAnyone \cite{xarchakos2024tryoffanyone}, Any2anyTryon \cite{guo2025any2anytryon}. As the inverse task of VTON, try-off needs to display the garment's full view and preserve all details, like color, logo, texture and texts. In \cref{tab:tryoff_vitonhd,tab:tryoff_dresscode,tab:tryoff_omnitryon}, our method remarkably outperforms all baselines on public and our benchmarks, confirming that our method takes good command of try-on and try-off abilities. Their detail preservation ability is mutually reinforcing and does not conflict. In \cref{fig:tryoff}, it is evident that the realism and fidelity of our try-off results are far superior to those of Any2anyTryon, verifying the validation of our unified model. More visualization examples can be found in \cref{fig:add_takeoff}.

\begin{figure}[ht]
  \centering
  \includegraphics[
      height=3.5cm, 
      page=18, 
      trim=2cm 7.5cm 1.8cm 1.4cm, 
      clip               
  ]{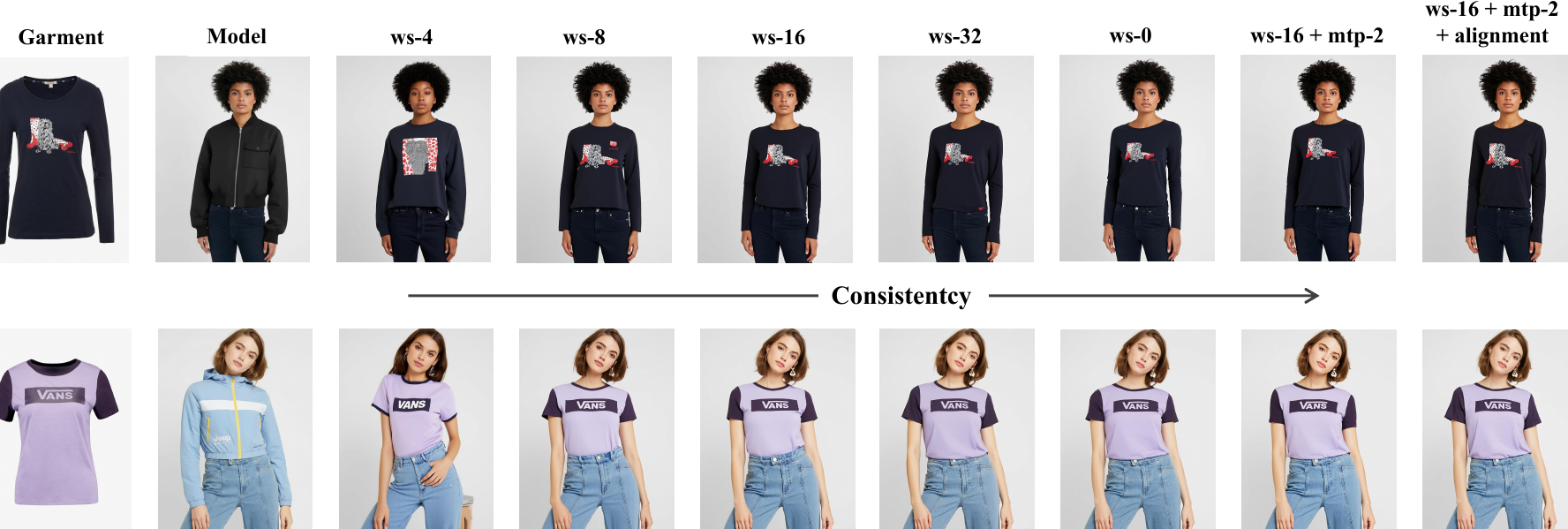}
  \caption{Qualitative comparisons on different window sizes and training strategies.
  }
  \label{fig:ablation}
\end{figure}

\begin{table}[ht]
\centering
\caption{Ablation study on \textbf{VITON-HD} dataset for the \textbf{model-based try-on} task. ref-512 + ws-16 + mtp-2 denotes that the reference condition is $512 \times 512$, local window size is 16 and 2 timesteps prediction.}
\label{tab:ablation}
\begin{tabular}{l|c|ccccc}
\toprule
Method & Alignment & FID$\downarrow$ & KID$\downarrow$ & SSIM$\uparrow$ & LPIPS$\downarrow$  \\
\midrule
ref-512 + ws-0 + mtp-1  & \ding{55} & 7.6341 & 1.2143 & 0.8607 & 0.1070 \\
ref-512 + ws-4 + mtp-1  & \ding{55} & 11.0544 & 2.6720 & 0.8531 & 0.1392  \\
ref-512 + ws-8 + mtp-1  & \ding{55} & 10.2492 & 2.6107 & 0.8680 & 0.1159\\
ref-512 + ws-16 + mtp-1 & \ding{55} & 8.5141 & 1.7236 & 0.8683 & 0.1090 \\
ref-512 + ws-32 + mtp-1 & \ding{55} & 8.3264 & 1.7037 & 0.8712 & 0.1095 \\
\midrule
ref-384 + ws-16 + mtp-2 & \ding{55} & 8.6728 & 1.3855 & 0.8738 & 0.1075 \\
ref-512 + ws-16 + mtp-2 & \ding{55} & 7.6120 & 0.9082 & 0.8781 & 0.0951  \\
ref-768 + ws-16 + mtp-2 & \ding{55} & 7.5431 & 0.8455 & 0.8800 & 0.0873  \\
\midrule
\textbf{ref-512 + ws-16 + mtp-2} & \ding{51} & \textbf{6.4564} & \textbf{0.7502} & \textbf{0.8838} & \textbf{0.0784} & \\
\bottomrule
\end{tabular}
\end{table}

\subsection{Ablation Study}
\label{sec:ablation}

We conduct ablation experiments to study the effect of Shift Window Attention, Multiple Timesteps Prediction and Alignment Loss. As shown in \cref{tab:ablation} and \cref{fig:ablation}, compared with original full attention (ws-0), a small window size will degrade performance, which is reasonable that the small window prevents attending further regions and limits the overall perspective. Because a larger window is approximately equivalent to the global latent size and the performance gap between $ws = 16$ and full attention is not intolerable, we set $ws = 16$ as the final choice. After employing MTP, four metrics have shown notable improvements, in which FID reduces from 8.2141 to 7.7120, and KID reduces from 1.8236 to 1.1282, indicating that MTP has positive effects on improving generation fidelity, due to achieving more stable and smooth trajectories, and lowering integration error, as analyzed in Appendix \ref{app:empir}. 

In \cref{tab:ablation}, we also compare different resolutions for the reference conditions, and the result of $384 < 512 < 768$ demonstrates that larger reference images will reduce the loss of detail information and ensure generation quality at the cost of a longer inference time. After adding the alignment loss, it is evident to see a significant performance improvement, and even the metrics surpass those of ref-768, which verifies the validation of alignment loss.

\subsection{Limitations}
\label{sec:limitations}
Despite the outstanding performance OmniDiT has achieved, we still face a major challenge: OmniDiT struggles to maintain the human model attributes, especially the human model's gesture. This mainly stems from that the inpainted human is slightly different from the tryon results in the data curation stage and not filtered out due to the VLM's weak visual ability, thus confusing the optimization objective in training afterwards. Therefore, research on how to develop a stronger VLM can benefit our data quality and model improvement, which is our future direction.

\section{Conclusion}
In this paper, we propose one unified VTON and VTOFF framework OmniDiT based on the Diffusion Transformer, demonstrating outstanding performance in three VTON and VTOFF tasks. To incorporate multiple reference conditions, we design an adaptive positive encoding to reduce token signal confusion induced by concatenating all tokens into one sequence. Meanwhile, we introduce the Shifted Window Attention to relieve the bottleneck of attention computation, thus achieving a linear computation complexity. Multiple timesteps prediction and alignment loss further improve the generation quality and fidelity. We also construct a large VTON dataset produced by our self-evolving data curation pipeline, aiming to address the data scarcity problem. Experiments reveal that our model has a significant advantage in generating high-fidelity and consistent try-on and try-off results under various complex scenes.



\section*{Acknowledgements}
Special thanks to Kuaishou Technology for supporting this research.
%
%
\bibliographystyle{splncs04}
\bibliography{main}

\newpage  
\appendix   
\crefalias{section}{appendix}  

\section{Dataset Details}
\label{app:dataset}

As illustrated in \cref{sec:dataset}, we carry out several rounds of data curation iterations, and add additional data samples from two public datasets \cite{choi2021viton, morelli2022dress} as a supplementary, thus obtaining a total of over 380k samples. Among all samples, each garment is labeled with one category and there are 23 categories in our dataset, as shown in \cref{fig:cate}. We keep the original garment labels for the public dataset samples. For model evaluation, we sample 1895 samples to establish the test set with stratified sampling based on the category distribution. Some data samples can be viewed in \cref{fig:data_samples}. In \cref{tab:dataset}, we compare our Omni-TryOn with other popular Try-On datasets, it is evident that our dataset is largest, most complete, and diverse. Our dataset equipped with detailed task instructions and matching human models, can make up the limitation of current try-on datasets and promote the VTON and VTOFF technology's development.

\begin{table}[h]
\caption{Comparison between our Omni-TryOn and related try-on datasets. 'Prompts' and 'Models' mean that dataset contains detailed text instructions for three tasks and matching models for the model-based try-on task.}
\label{tab:dataset}
\centering
\begin{tabular}{cccccc}
\toprule
Dataset & Prompts & Models &  \#Garments & \#Pairs & Resolution \\
\midrule
TryOnGAN \cite{lewis2021tryonganbodyawaretryonlayered} &  \ding{55} & \ding{55} & 52,000 & 52,000 & $512 \times 512$ \\
Revery AI\cite{li2021accuraterealisticoutfitsvisualization} &  \ding{55} & \ding{55} & 321,000 & 321,000 & $512 \times 512$ \\
VITON-HD \cite{choi2021viton} &  \ding{55} & \ding{55} & 13,679 & 13,679 & $1024 \times 768$ \\
DressCode \cite{morelli2022dress} &  \ding{55} & \ding{55}  & 53,792 & 53,792 & $1024 \times 768$ \\
IGPair \cite{shen2025imagdressing} & \ding{51} & \ding{55} & 86,873 & 324,857 & $>2K \times 2K$ \\
\midrule
\textbf{Omni-TryOn}(Ours) & \ding{51} & \ding{51} & 335,146 & 382,223 & diverse \\
\bottomrule
\end{tabular}
\end{table}

\begin{figure}[ht]
  \centering
  \begin{subfigure}[b]{0.45\textwidth} 
    \centering
    \includegraphics[
        height=4.5cm, 
        page=1, 
    ]{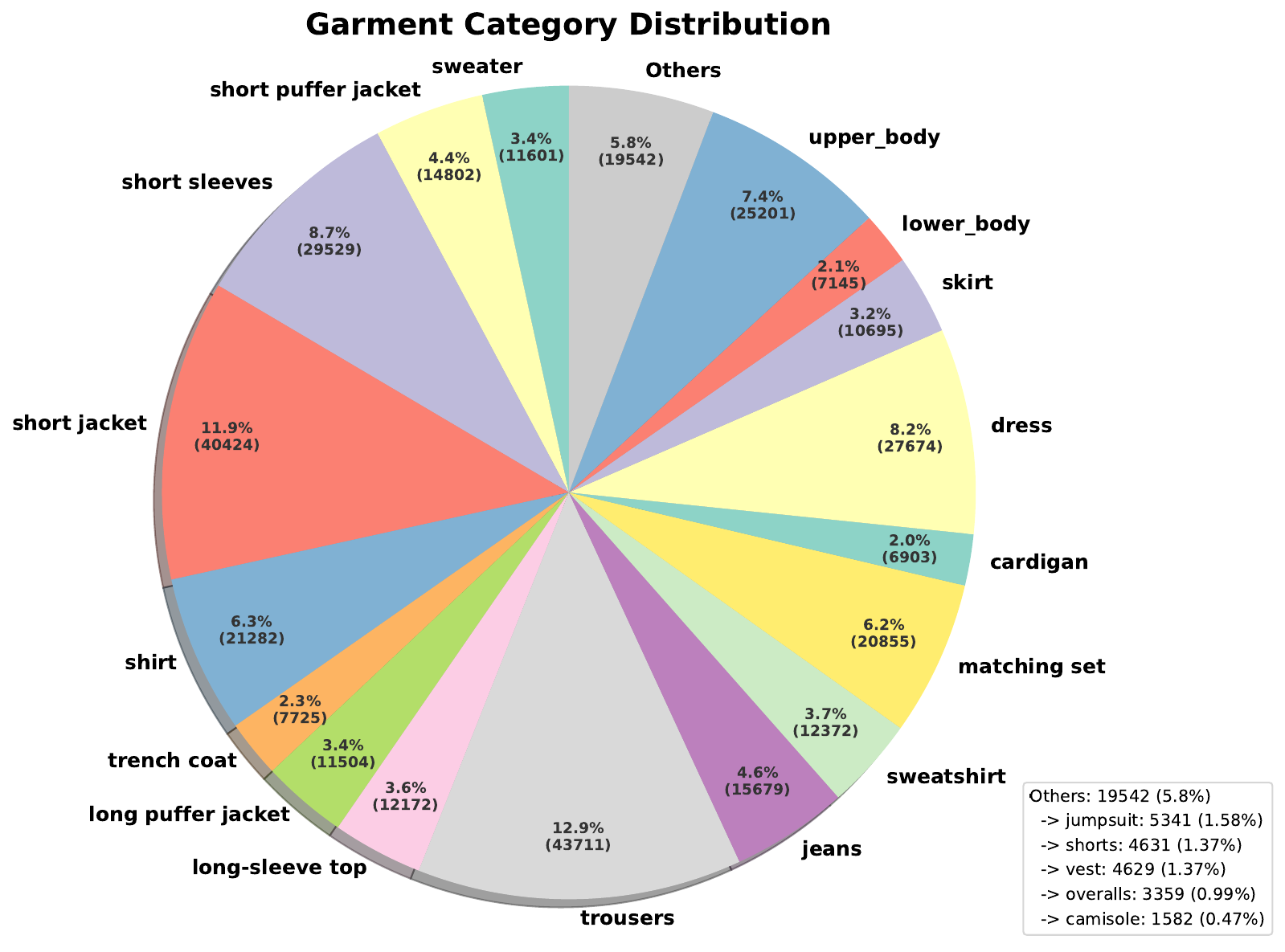}
    \caption{Training set} 
    \label{fig:cate-a}
  \end{subfigure}
  \hfill 
  \begin{subfigure}[b]{0.45\textwidth}
    \centering
    \includegraphics[
        height=4.5cm, 
        page=1,
    ]{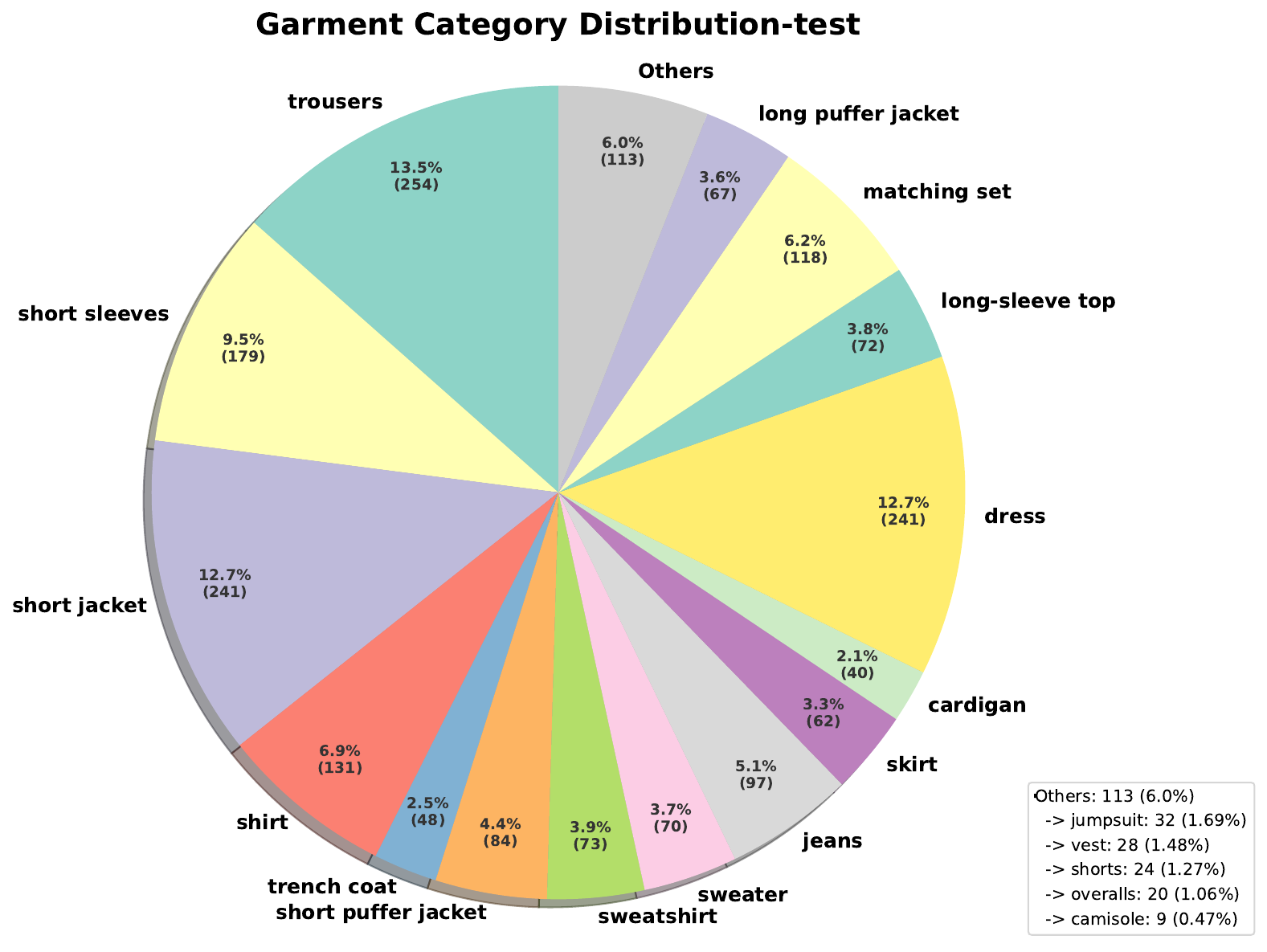} 
    \caption{Test set} 
    \label{fig:cate-b}
  \end{subfigure}
  
  \caption{Garment category distribution of our Omni-TryOn dataset.} 
  \label{fig:cate}
\end{figure}

\begin{figure}[h]
  \centering
  \includegraphics[
      height=6.5cm, 
      page=19, 
      trim=0.5cm 0.5cm 1.8cm 0cm, 
      clip               
  ]{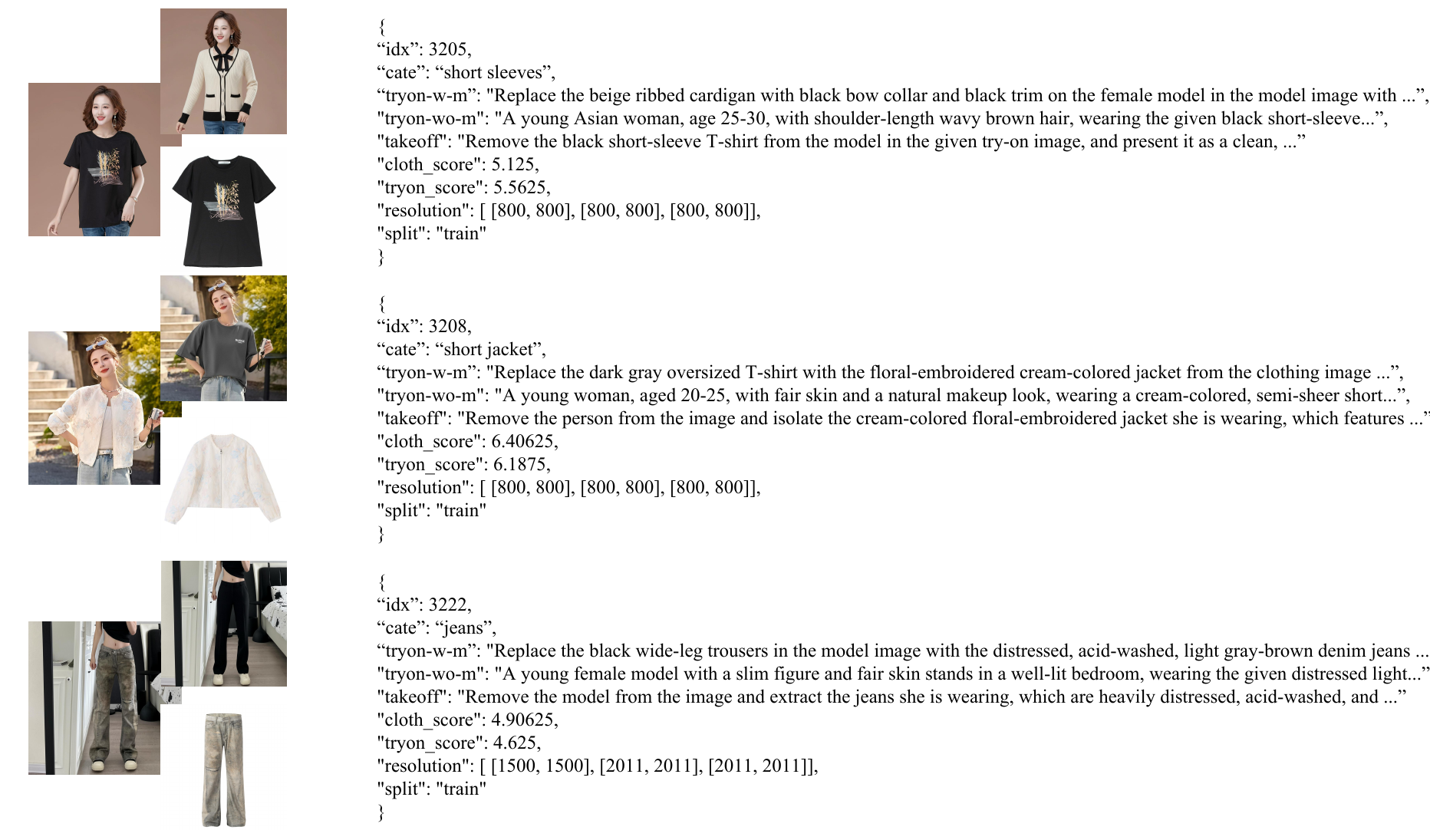}
  \caption{Some data sample showcases in our dataset Omni-TryOn.
  }
  \label{fig:data_samples}
\end{figure}

\section{Experiment Settings}
\label{app:setting}

\subsection{Model Architecture}
We replace all the attention modules in MM-DiT blocks (including double-stream and single-stream) with our introduced shifted window attention, and shift the windows every two layers. The window size is 16 and the shift size is 8. All local window attention computations are applied to the reference images.

\subsection{Training Settings}
Inspired by UNO \cite{wu2025less}, with the aim of progressive cross-modal alignment, we adopt a two-stage training approach. For the first stage, we train the model using single-reference pair data (model-free try-on and try-off tasks) for one epoch. Then, we continue training on mixed reference pair data (all three tasks) for 2-3 epochs. For public datasets, both stages utilize the learning rate of $4e-5$ and a batch size of 2 for each GPU (a total batch size of 16). For our own dataset, both stages utilize the learning rate of $4e-5$, a batch size of 2 for each GPU (a total batch size of 16) for the first stage and a batch size of 4 for each GPU (a total batch size of 32) for the second stage. In all trainings, we use the AdamW optimizer \cite{loshchilov2017fixing}.

For public datasets, we resize the reference images to $512 \times 768$ as input, and generate the target images with a resolution of $768 \times 1024$. And for our own dataset Omni-TryOn, the reference images are $512 \times 512$ and target images are $768 \times 768$.

For MTP, we predict two timesteps for every iteration, and the time interval $\Delta t$ is 30 in training.

To ensure the asthetics of generated garments and tryon results, we sample some proportions of data whose garment and tryon images' asthetic scores are larger than 4.5 for model-free try-on and try-off tasks, and the rest of the datasets for model-based try-on task to improve the consistency.

\subsection{Batch Sampling Strategy}
Considering the mixed number of reference images in the second training stage (one reference image for model-free try-on and try-off tasks, and two reference images for model-based try-on task), we strictly restrict that one batch has the same number of reference images. In other words, in all optimization steps, the model-based try-on samples are isolated from the other two tasks.

\section{Theoretical Analysis}
\label{app:analysis}

\subsection{Flow Matching Framework}

Let $p_0$ and $p_1$ be probability distributions over $\mathbb{R}^d$, representing the source and target distributions respectively. Flow Matching seeks to learn a velocity field $v_\theta: \mathbb{R}^d \times [0,1] \to \mathbb{R}^d$ such that the solution to the initial value problem:
\begin{equation}
    \frac{dx}{dt} = v_\theta(x, t), \quad x(0) \sim p_0
\end{equation}
satisfies $x(1) \sim p_1$. In practice, the velocity field is parameterized by a neural network and trained using a conditional flow matching objective.

For a pair of samples $(x_0, x_1)$ where $x_0 \sim p_0$ and $x_1 \sim p_1$, the conditional flow is defined as:
\begin{equation}
    \phi_t(x_0, x_1) = (1-t)x_0 + t x_1
\end{equation}
and the corresponding conditional velocity field is:
\begin{equation}
    u_t(x | x_0, x_1) = x_1 - x_0
\end{equation}

The standard Flow Matching loss minimizes the discrepancy between the learned velocity $v_\theta$ and the conditional velocity $u_t$:
\begin{equation}
    \mathcal{L}_{\text{FM}}(\theta) = \mathbb{E}_{\substack{t \sim \mathcal{U}[0,1] \\ (x_0, x_1) \sim p_0 \times p_1}} \left[ \| v_\theta(\phi_t(x_0, x_1), t) - (x_1 - x_0) \|^2 \right]
\end{equation}

\subsection{Single-Step vs. Multi-Timestep Prediction}

\textbf{Single-Step Prediction (SSP)} trains the model by sampling a single random time $t$ and minimizing the velocity prediction error at that time point:
\begin{equation}
    \mathcal{L}_{\text{SSP}} = \mathbb{E}_{t, x_t} \left[ \| v_\theta(x_t, t) - (x_1 - x_0) \|^2 \right]
\end{equation}

\noindent
\textbf{Multi-Timestep Prediction (MTP)} extends this by unrolling $K - 1$ Euler integration steps from time $t$ to $t - (K - 1)\Delta t$, supervising the velocity prediction at each intermediate step:
\begin{equation}
    \mathcal{L}_{\text{MTP}} = \frac{1}{K} \sum_{k=0}^{K-1} \mathbb{E} \left[ \| v_\theta(x_{t_k}, t_k) - (x_1 - x_0) \|^2 \right]
\end{equation}
where $x_{t_{k+1}} = x_{t_k} + (t_{k+1} - t_k) \cdot v_\theta(x_{t_k}, t_k)$ and $t_k = t - k\Delta t$.

\subsection{Lipschitz Continuity and Velocity Field Smoothness}

A key property of well-behaved velocity fields is Lipschitz continuity, which ensures stable numerical integration.

\begin{definition}[Lipschitz Continuity]
    A velocity field $v: \mathbb{R}^d \times [0,1] \to \mathbb{R}^d$ is $L$-Lipschitz continuous if:
    \begin{equation}
        \| v(x, t) - v(x', t') \| \leq L (\| x - x' \| + |t - t'|)
    \end{equation}
    for all $x, x' \in \mathbb{R}^d$ and $t, t' \in [0,1]$.
\end{definition}

The Lipschitz constant $L$ quantifies the smoothness of the velocity field. \textbf{Smaller $L$ implies smoother trajectories and lower numerical integration error.}

\subsection{Theoretical Results}
\label{app:theoret}

We now present our main theoretical results establishing the advantages of MTP.

\begin{theorem}[Implicit Smoothness Regularization]
    \label{thm:smoothness}
    Multi-Step Prediction implicitly regularizes the temporal smoothness of the velocity field. Specifically, the MTP loss can be decomposed as:
    \begin{equation}
        \mathcal{L}_{\text{MTP}} = \mathcal{L}_{\text{SSP}} + \lambda \cdot \mathcal{R}_{\text{smooth}} + \mathcal{O}(\Delta t^2)
    \end{equation}
    where $\lambda > 0$ and the smoothness regularizer is:
    \begin{equation}
        \mathcal{R}_{\text{smooth}} = \mathbb{E} \left[ \sum_{k=0}^{K-2} \| v_\theta(x_{t_{k+1}}, t_{k+1}) - v_\theta(x_{t_k}, t_k) \|^2 \right]
    \end{equation}
\end{theorem}

\begin{proof}
    Consider two consecutive time steps $t_k$ and $t_{k+1} = t_k - \Delta t$. By the triangle inequality:
    \begin{align}
        &\| v_\theta(x_{t_{k+1}}, t_{k+1}) - v_\theta(x_{t_k}, t_k) \| \nonumber \\
        \leq\ & \| v_\theta(x_{t_{k+1}}, t_{k+1}) - u_{t_{k+1}} \| + \| u_{t_{k+1}} - u_{t_k} \| + \| u_{t_k} - v_\theta(x_{t_k}, t_k) \| \nonumber \\
        \leq\ & \| v_\theta(x_{t_{k+1}}, t_{k+1}) - u_{t_{k+1}} \| + \| u_{t_k} - v_\theta(x_{t_k}, t_k) \| + \mathcal{O}(\Delta t)
    \end{align}
    where $u_t = x_1 - x_0$ is the ground-truth conditional velocity, and we used the fact that $u_t$ is constant with respect to $t$.
    
    Squaring both sides and taking expectations:
    \begin{align}
        &\mathbb{E} \left[ \| v_\theta(x_{t_{k+1}}, t_{k+1}) - v_\theta(x_{t_k}, t_k) \|^2 \right] \nonumber \\
        \leq\ & 2 \mathbb{E} \left[ \| v_\theta(x_{t_{k+1}}, t_{k+1}) - u_{t_{k+1}} \|^2 \right] + 2 \mathbb{E} \left[ \| v_\theta(x_{t_k}, t_k) - u_{t_k} \|^2 \right] + \mathcal{O}(\Delta t^2) \nonumber \\
        =\ & 2 \mathcal{L}_{\text{SSP}}(t_{k+1}) + 2 \mathcal{L}_{\text{SSP}}(t_k) + \mathcal{O}(\Delta t^2)
    \end{align}
    
    Summing over all consecutive pairs and rearranging:
    \begin{equation}
        \mathcal{L}_{\text{MTP}} = \frac{1}{K} \sum_{k=0}^{K-1} \mathcal{L}_{\text{SSP}}(t_k) \geq \mathcal{L}_{\text{SSP}} + \frac{1}{2K} \mathcal{R}_{\text{smooth}} - \mathcal{O}(\Delta t^2)
    \end{equation}
    which completes the proof with $\lambda = \frac{1}{2K}$.
\end{proof}

\begin{remark}
    Theorem \ref{thm:smoothness} shows that MTP penalizes large changes in the velocity field between adjacent time steps. This encourages the model to learn a temporally consistent velocity field with smaller Lipschitz constant.
\end{remark}

\begin{theorem}[Trajectory Integration Error Bound]
    \label{thm:error_bound}
    Let $v_\theta$ be an $L_\theta$-Lipschitz velocity field learned via Flow Matching, and let $x_0^{\text{gen}}$ be the result of numerically integrating the ODE $\frac{dx}{dt} = v_\theta(x, t)$ from $x_1$ to $t=0$ using Euler's method with step size $\Delta t$. Then the trajectory error satisfies:
    \begin{equation}
        \| x_0^{\text{gen}} - x_0 \| \leq C \cdot L_\theta \cdot \Delta t + \mathcal{O}(\Delta t^2)
    \end{equation}
    where $C$ is a constant independent of $L_\theta$ and $\Delta t$.
\end{theorem}

\begin{proof}
    The global truncation error of Euler's method for an $L$-Lipschitz ODE is bounded by \cite{iserles2009first}:
    \begin{equation}
        \| x(t) - x_{\text{num}}(t) \| \leq \frac{M}{2L} (e^{Lt} - 1) \Delta t
    \end{equation}
    where $M$ bounds the second derivative of the true solution. For our linear conditional flow $\phi_t(x_0, x_1) = (1-t)x_0 + t x_1$, we have $M = 0$, but the error arises from the discrepancy between $v_\theta$ and the true velocity $u_t$.
    
    Let $\epsilon(t) = v_\theta(x_t, t) - u_t$ be the velocity prediction error. The accumulated error after $N = 1/\Delta t$ steps is:
    \begin{align}
        \| x_0^{\text{gen}} - x_0 \| &= \left\| \sum_{k=0}^{N-1} \left[ v_\theta(x_{t_k}, t_k) - u_{t_k} \right] \Delta t \right\| \nonumber \\
        &\leq \sum_{k=0}^{N-1} \| \epsilon(t_k) \| \Delta t
    \end{align}
    
    By the Lipschitz property of $v_\theta$ and the fact that $u_t$ is constant:
    \begin{align}
        \| \epsilon(t_{k+1}) - \epsilon(t_k) \| &= \| v_\theta(x_{t_{k+1}}, t_{k+1}) - v_\theta(x_{t_k}, t_k) \| \nonumber \\
        &\leq L_\theta (\| x_{t_{k+1}} - x_{t_k} \| + \Delta t) \nonumber \\
        &\leq L_\theta (\| v_\theta \| \Delta t + \Delta t) \nonumber \\
        &\leq L_\theta C' \Delta t
    \end{align}
    where $C' = \| v_\theta \| + 1$.
    
    This implies that the error $\epsilon(t)$ cannot change too rapidly, and the accumulated error is bounded by:
    \begin{equation}
        \| x_0^{\text{gen}} - x_0 \| \leq C \cdot L_\theta \cdot \Delta t + \mathcal{O}(\Delta t^2)
    \end{equation}
    where $C$ absorbs constants related to $\| v_\theta \|$ and the number of steps.
\end{proof}

\begin{corollary}
    If MTP reduces the Lipschitz constant of the learned velocity field such that $L_\theta^{\text{MTP}} < L_\theta^{\text{SSP}}$, then:
    \begin{equation}
        \| x_0^{\text{gen, MTP}} - x_0 \| < \| x_0^{\text{gen, SSP}} - x_0 \|
    \end{equation}
    for sufficiently small $\Delta t$.
\end{corollary}

\subsection{Empirical Results}
\label{app:empir}
We employ the SSP and MTP model weights to generate 1500 try-on results under the same settings. As illustrated in \cref{fig:lipschi}, the Lipschitz constants of the MTP model are apparently lower than those of the SSP model, which corresponds to the conclusion of \cref{thm:error_bound}.

Above all, we demonstrate that our MTP training strategy can enhance the smoothness of trajectory prediction and achieve lower integration error from both theoretical and empirical results.

\begin{figure}[ht]
  \centering
  \includegraphics[
      height=7.5cm, 
      page=1, 
  ]{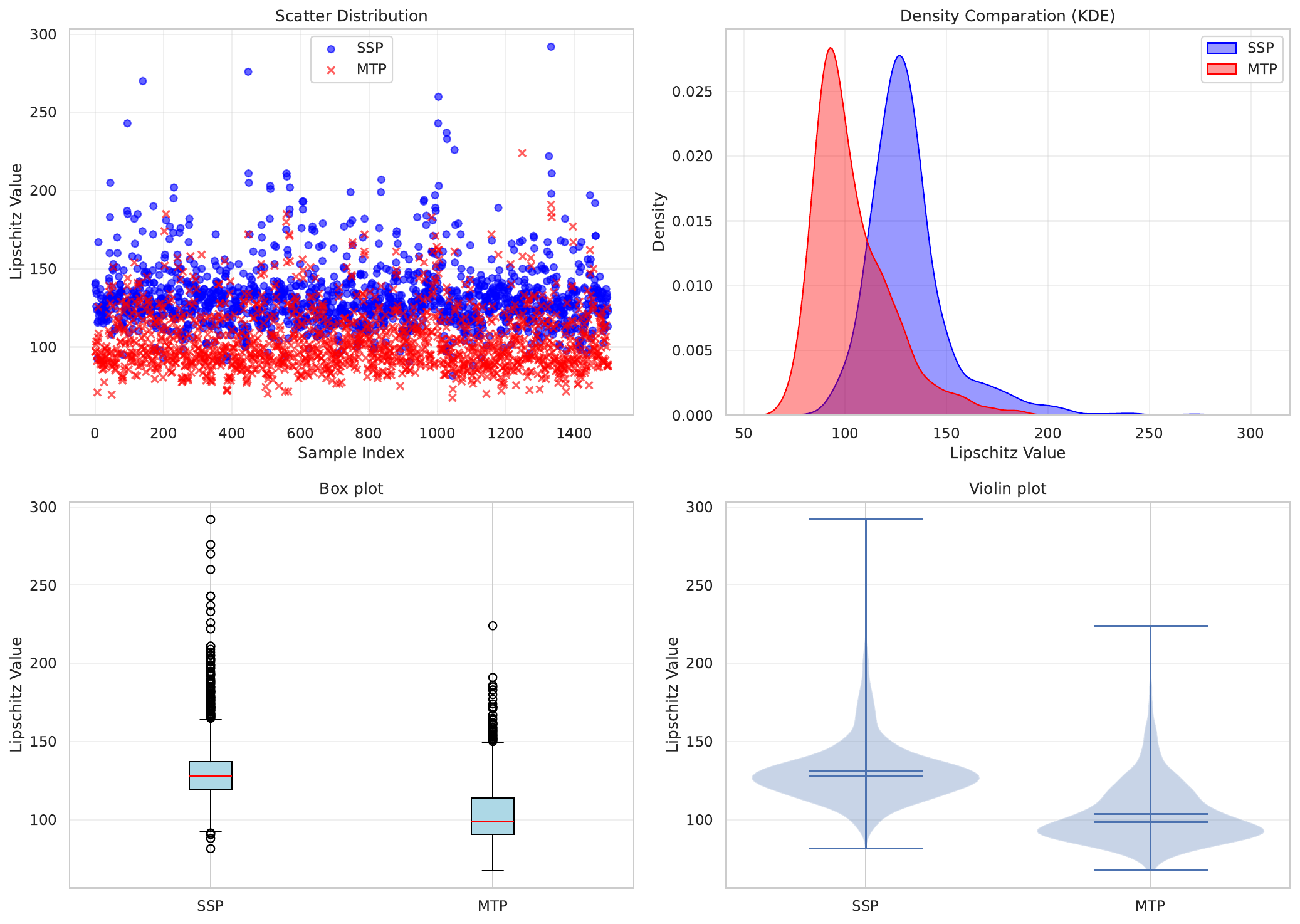}
  \caption{Lipschitz constant comparison between SSP and MTP.
  }
  \label{fig:lipschi}
\end{figure}

\section{Additional Results}
\label{app:add_results}

\begin{table}[ht]
\centering
\caption{Quantitative comparison on \textbf{Omni-TryOn} dataset for the \textbf{model-based try-on} task. We multiply KID by 1000 for better comparison. The best and the second best results are denoted as \textbf{Bold} and \underline{underline}, respectively.}
\label{tab:tryon_w_m_omni_tryon}
\begin{tabular}{l|cccc}
\toprule
Dataset & \multicolumn{4}{c}{Omni-TryOn}\\
\cmidrule(lr){2-5}
Method
 & FID$\downarrow$ & KID$\downarrow$ & SSIM$\uparrow$ & LPIPS$\downarrow$  \\
\midrule

CatVTON\cite{chong2024catvton} &36.1327 & 15.3471 & 0.5904 & 0.3497\\
FitDiT\cite{jiang2024fitdit} & 25.1988 & 7.0999 & 0.6673 & \underline{0.2440} \\
Any2anyTryon\cite{guo2025any2anytryon} & 25.9283 & 9.7514 & 0.6537 & 0.2954 \\
Jco-MVTON\cite{wang2025jco} & 27.5184 & 7.9581 & \underline{0.6702} & 0.2449 \\

\midrule

DreamOmni2\cite{xia2025dreamomni2} & \underline{20.7575} & \underline{3.0112} & 0.6444 & 0.2645\\
DreamO\cite{mou2025dreamo} & 32.6664 & 9.6479 & 0.4688 & 0.5072\\
UNO\cite{wu2025less} &  31.3706 & 9.0984 & 0.5470 & 0.4775 \\
OmniGen2\cite{wu2025omnigen2} & 25.1083 & 4.7097 & 0.5277 & 0.4315\\

\midrule

\textbf{OmniDiT} & \textbf{15.9592}  &  \textbf{1.4308}  &  \textbf{0.7115}  &  \textbf{0.1728} \\
\bottomrule
\end{tabular}
\end{table}

\begin{table}[ht]
\centering
\caption{Quantitative comparison on \textbf{DressCode} dataset for the \textbf{model-free try-on} task. The best and the second best results are denoted as \textbf{Bold} and \underline{underline}, respectively.}
\label{tab:tryon_wo_m_dresscode}
\begin{tabular}{l|ccccccc}
\toprule
Method
 & DINO-I $\uparrow$ &  CLIP-I $\uparrow$ & SSIM $\uparrow$  &  LPIPS $\downarrow$ & FID $\downarrow$ & KID $\downarrow$ & FFA $\uparrow$\\
\midrule
MagicClothing\cite{chen2024magic} & 0.3652 &  0.7884  &  0.6137  &  0.4831  &  29.7541  &  22.0792 & 0.5872 \\
StableGarment\cite{wang2024stablegarment} & 0.3687  &  0.7537  &  0.5636  &  0.5465  & 57.6464  &  37.0746 & 0.5941 \\
Any2anyTryon\cite{guo2025any2anytryon} & 0.3269  & 0.7998  & 0.7083  & 0.4004  & 56.2823  & 35.8366 & 0.5407 \\
DreamFit\cite{lin2025dreamfit} & \underline{0.4087}  & 0.8253  & 0.7486  & \underline{0.2574}  &  \underline{14.6571} & \underline{6.5216} & \underline{0.5962} \\
IMAGDressing\cite{shen2025imagdressing} & 0.4049  & \underline{ 0.8302} &  \underline{0.7713} & 0.2848  & 16.1254  & 8.7264 & 0.5879  \\

\midrule

\textbf{OmniDiT} & \textbf{0.4336}  &  \textbf{0.8533}  &  \textbf{0.7758}  &  \textbf{0.2124}  &  \textbf{9.6294}  &  \textbf{3.0329}  &  \textbf{0.6103}   \\
\bottomrule
\end{tabular}
\end{table}

\begin{table}[ht]
\centering
\caption{Quantitative comparison on \textbf{Omni-TryOn} dataset for the \textbf{model-free try-on} task. The best and the second best results are denoted as \textbf{Bold} and \underline{underline}, respectively.}
\label{tab:tryon_wo_m_omnitryon}
\begin{tabular}{l|ccccccc}
\toprule
Method
 & DINO-I $\uparrow$ &  CLIP-I $\uparrow$ & SSIM $\uparrow$  &  LPIPS $\downarrow$ & FID $\downarrow$ & KID $\downarrow$ & FFA $\uparrow$\\
\midrule
MagicClothing\cite{chen2024magic} & 0.3721 & 0.7667 & 0.4847 & 0.5448 & 49.2383 & 20.7358 & 0.5514 \\
StableGarment\cite{wang2024stablegarment} & 0.3949 & 0.7551 & 0.4582 &0.5604 & 61.8785 & 22.3345 & 0.5695  \\
Any2anyTryon\cite{guo2025any2anytryon} & 0.3701 & \textbf{0.8014} & 0.5409 & 0.5185 & 84.2664 & 53.0184 & 0.5539  \\
DreamFit\cite{lin2025dreamfit} & 0.3777 & 0.7388 & 0.5205 & 0.4567 & 29.1108 & 6.6340 & 0.5539 \\
IMAGDressing\cite{shen2025imagdressing} &  0.3900 & \underline{0.7930} & 0.5361 & 0.5297 & 41.0636 & 15.2807 & 0.5615 \\

\midrule

DreamOmni2\cite{xia2025dreamomni2} & 0.3685 & 0.6963 & 0.4922 & 0.4769 & 38.7223 & 15.5396 & 0.5593 \\
 OminiControl\cite{song2025omniconsistency} & 0.3606 & 0.7324 & 0.5144 & 0.4844 & 31.4160 & 9.7545 & 0.5476 \\
DreamO\cite{mou2025dreamo} & 0.4092 & 0.7074 & 0.5075 & \underline{0.4553} & 30.3898 & 8.3570 & 0.6028\\
UNO\cite{wu2025less} & \underline{0.4100} & 0.7639 & \underline{0.5499} & 0.4716 & 29.7925 & 7.6892 & 0.5923 \\
OmniGen2\cite{wu2025omnigen2} & \textbf{0.4217} & 0.7315 & 0.5090 & 0.4705 & \underline{26.1271} & \underline{4.3034} & \underline{0.6120}  \\

\midrule

\textbf{OmniDiT} & 0.3822  &  0.7417  &  \textbf{0.5658}  &  \textbf{0.3878}  &  \textbf{22.7795}  &  \textbf{3.7496}  & \textbf{0.6180}  \\
\bottomrule
\end{tabular}
\end{table}

\begin{table}[h]
\centering
\caption{Quantitative comparison on \textbf{DressCode} dataset for the \textbf{try-off} task. The best and the second best results are denoted as \textbf{Bold} and \underline{underline}, respectively.}
\label{tab:tryoff_dresscode}
\begin{tabular}{l|ccccccc}
\toprule
Method
 & CLIP-I $\uparrow$ & MS-SSIM $\uparrow$ &  LPIPS $\downarrow$ & FID $\downarrow$ & CLIP-FID $\downarrow$ & KID $\downarrow$  & DISTS $\downarrow$ \\
\midrule
TryOffDiff\cite{velioglu2024tryoffdiff} & 0.9257   &  \textbf{0.7432}  &  \underline{0.1979}  &  11.7073  &  8.2309 & \underline{3.4271} & 0.2480   \\
MGT\cite{velioglu2025mgt} & \underline{0.9259}   &  \underline{0.7377}  &  0.2022  &  \underline{11.1108}  &  \underline{8.0411} & 3.6770 & \underline{0.2476}   \\
TryOffAnyone\cite{xarchakos2024tryoffanyone} & -   &  - &  - & -  &  - & - & -   \\
Any2anyTryon\cite{guo2025any2anytryon} &  0.8586  & 0.5572  & 0.3635  & 22.9081  & 10.6153 & 7.6255 & 0.3399  \\

\midrule

\textbf{OmniDiT} & \textbf{0.9365}  &  0.7052  &  \textbf{0.1746}  &  \textbf{7.5019}  &  \textbf{2.3383}  &  \textbf{1.2174}  &  \textbf{0.2013}   \\
\bottomrule
\end{tabular}
\end{table}

\begin{table}[h]
\centering
\caption{Quantitative comparison on \textbf{Omni-TryOn} dataset for the \textbf{try-off} task. The best and the second best results are denoted as \textbf{Bold} and \underline{underline}, respectively.}
\label{tab:tryoff_omnitryon}
\begin{tabular}{l|ccccccc}
\toprule
Method
 & CLIP-I $\uparrow$ & MS-SSIM $\uparrow$ &  LPIPS $\downarrow$ & FID $\downarrow$ & CLIP-FID $\downarrow$ & KID $\downarrow$  & DISTS $\downarrow$ \\
\midrule
TryOffDiff\cite{velioglu2024tryoffdiff} &  0.8606 & \underline{0.5745} & 0.3884 & 34.3007 & 9.7895 & 11.1588 & 0.3334   \\
MGT\cite{velioglu2025mgt} & 0.8722 &  0.5698 & 0.3798 & 28.0170 & 8.4188 & 7.9347 & 0.3127  \\
TryOffAnyone\cite{xarchakos2024tryoffanyone} & 0.8472 & 0.4658 & 0.4169 & 46.5044 & 12.4310 & 20.9535 & 0.3061  \\
Any2anyTryon\cite{guo2025any2anytryon} &  0.8102  & 0.4728 & 0.4366 & 68.1454 & 16.6143 & 36.3953 & 0.3406 \\

\midrule

DreamOmni2\cite{xia2025dreamomni2} & \underline{0.9107}  & 0.4888 & 0.3600 & \underline{23.2848} & \underline{4.3091} & \underline{5.2890} & \underline{0.2701} \\
OminiControl\cite{song2025omniconsistency} & 0.8853  & 0.4700 & 0.3796 & 25.6226 & 5.5035 & 5.9275 & 0.2858 \\
DreamO\cite{mou2025dreamo} &0.8755 & 0.4825 & \underline{0.3577} & 30.7476 & 9.6856 & 10.0566 & 0.2809 \\
UNO\cite{wu2025less} & 0.8709 & 0.5052 & 0.3834 & 33.7029 & 11.5271 & 11.3796 & 0.2880 \\
OmniGen2\cite{wu2025omnigen2} & 0.8894  & 0.4651 & 0.4008 & 27.5556 & 5.9572 & 6.5501 & 0.2899  \\

\midrule

\textbf{OmniDiT} &  \textbf{0.9315} &  \textbf{0.6225}  &  \textbf{0.2367}  &  \textbf{13.3311}  &  \textbf{2.1556}  &  \textbf{0.7115}  &  \textbf{0.2035}   \\
\bottomrule
\end{tabular}
\end{table}

\begin{figure}[h]
  \centering
  \includegraphics[
      height=6.5cm, 
      page=10, 
      trim=0.5cm 0.5cm 1cm 0cm, 
      clip               
  ]{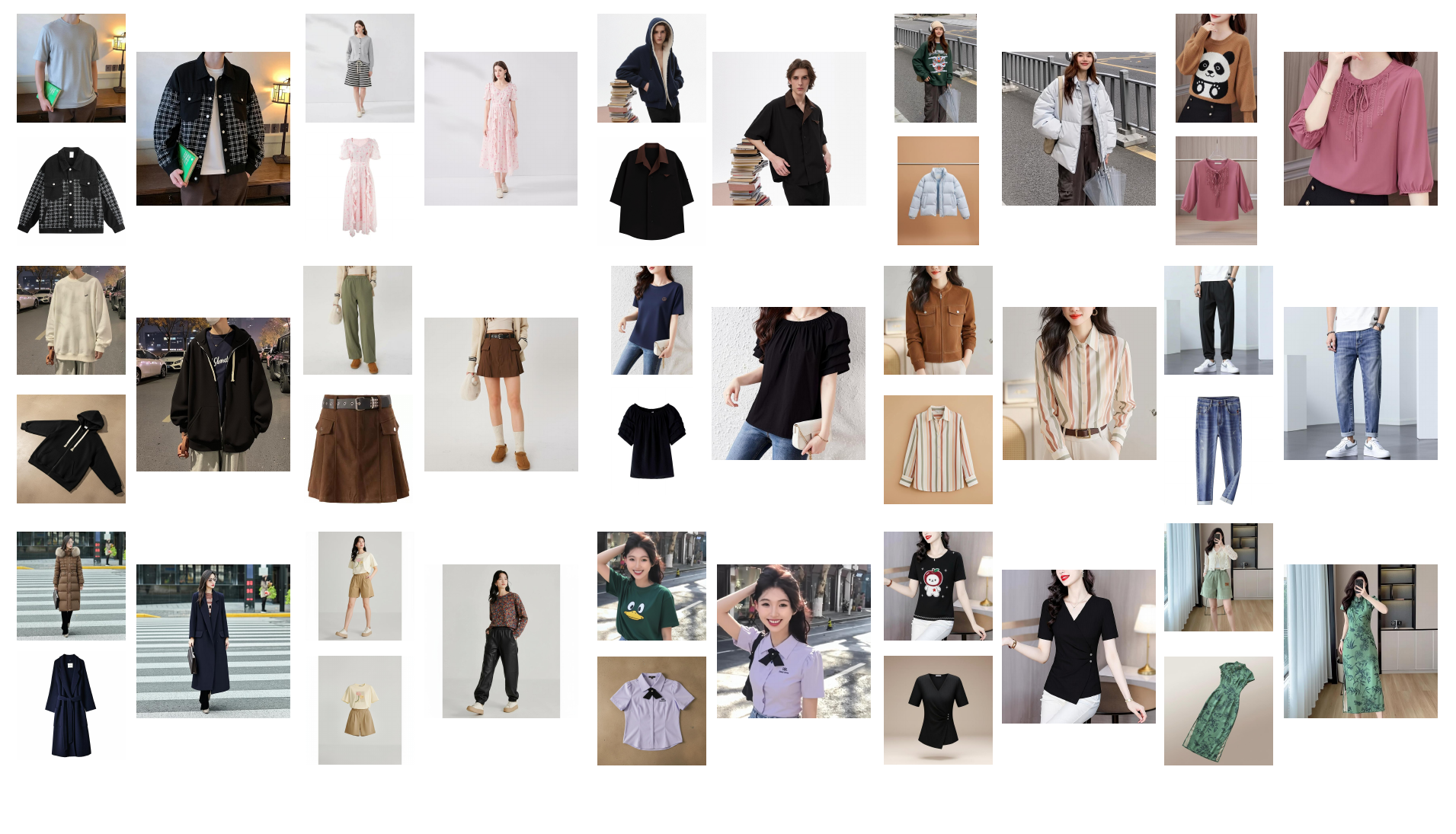}
  \caption{Additional model-based VTON showcases(1) in our dataset Omni-TryOn.
  }
  \label{fig:add_tryon_w_m_1}
\end{figure}

\begin{figure}[h]
  \centering
  \includegraphics[
      height=5.5cm, 
      page=11, 
      trim=0.5cm 4cm 1cm 0cm, 
      clip               
  ]{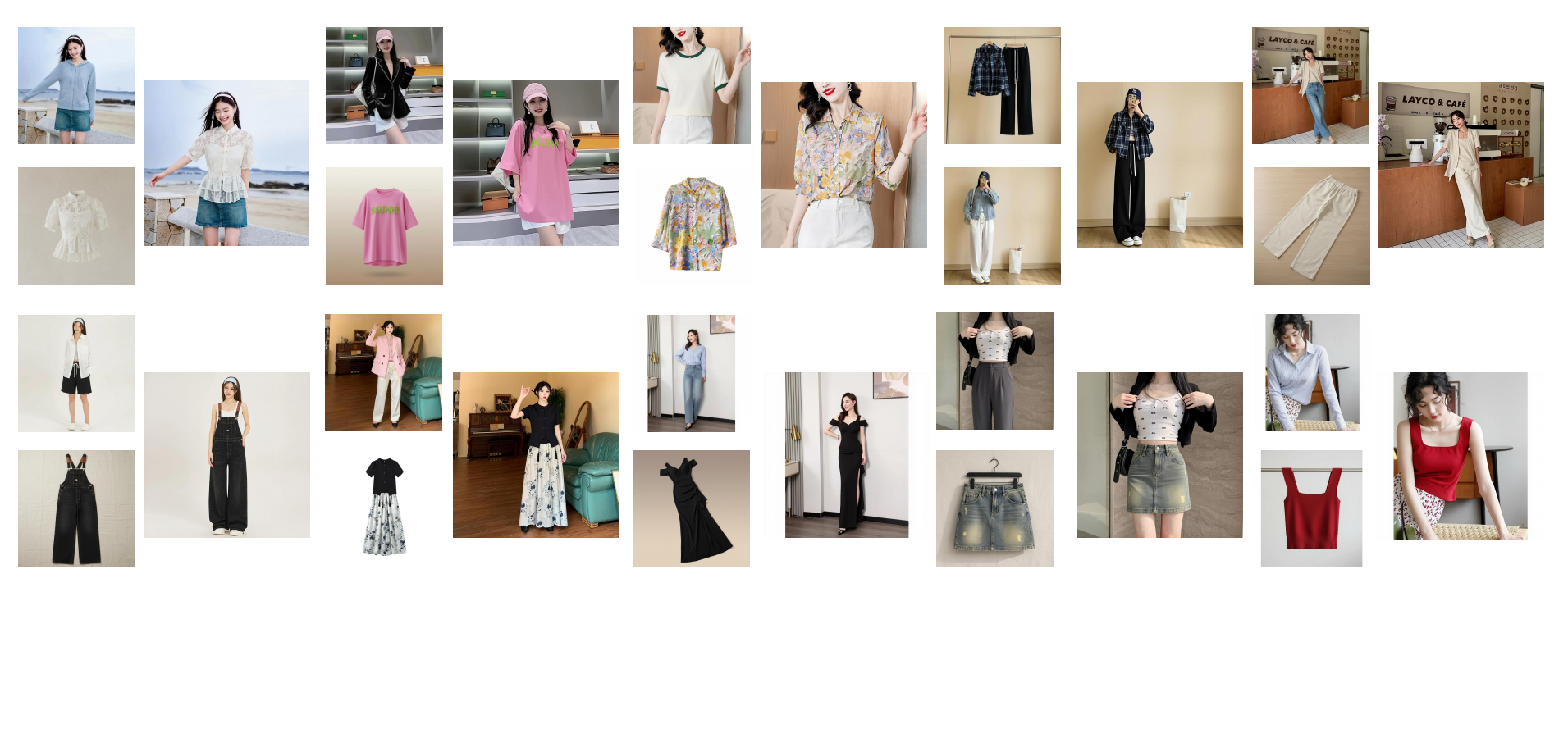}
  \caption{Additional model-based VTON showcases(2) in our dataset Omni-TryOn.
  }
  \label{fig:add_tryon_w_m_2}
\end{figure}

\begin{figure}[h]
  \centering
  \includegraphics[
      height=6.5cm, 
      page=9, 
      trim=0.5cm 0.5cm 0.5cm 0cm, 
      clip               
  ]{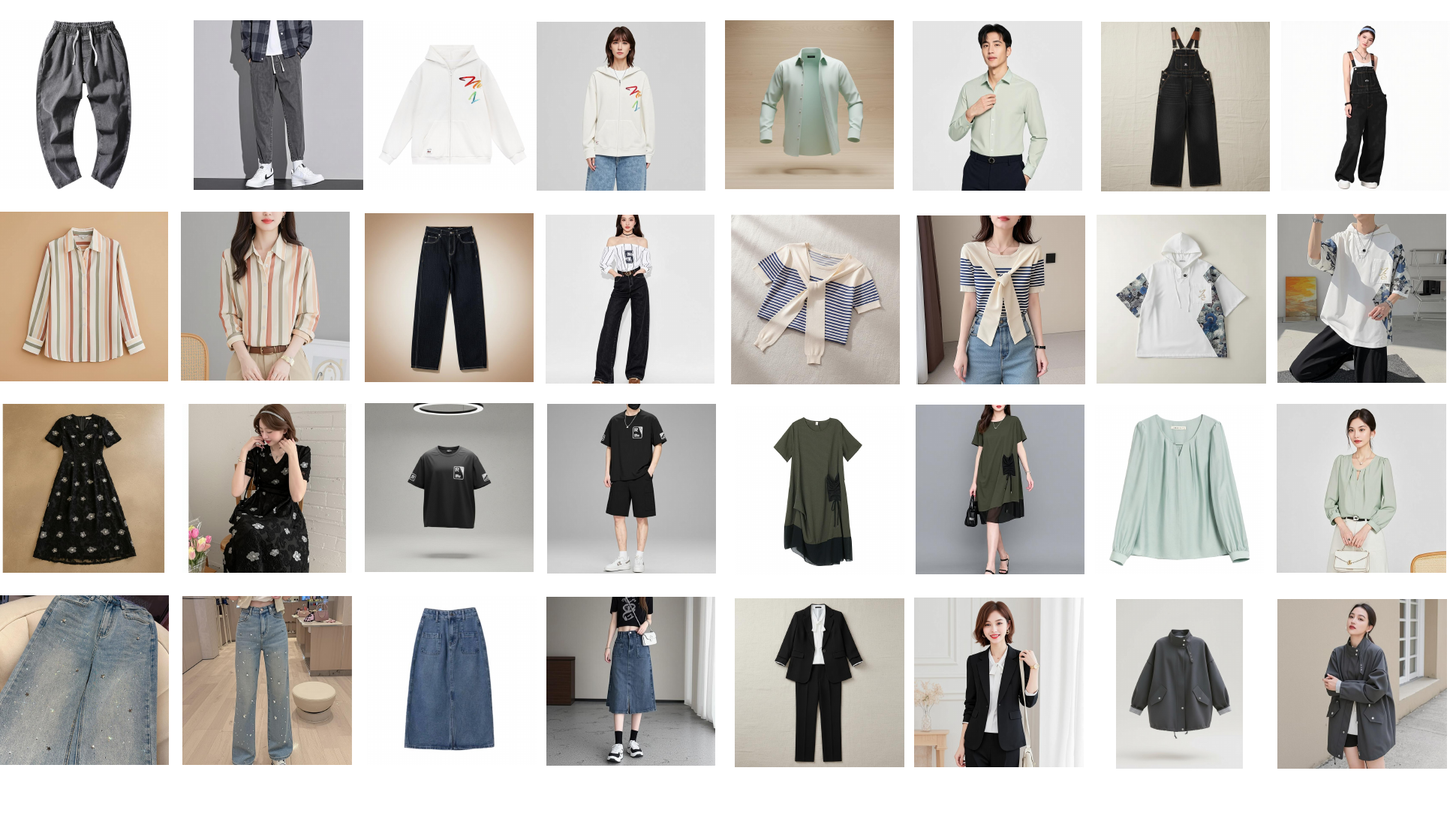}
  \caption{Additional model-free VTON showcases in our dataset Omni-TryOn.
  }
  \label{fig:add_tryon_wo_m}
\end{figure}

\begin{figure}[h]
  \centering
  \includegraphics[
      height=6.5cm, 
      page=8, 
      trim=0.5cm 0.5cm 0.2cm 0cm, 
      clip               
  ]{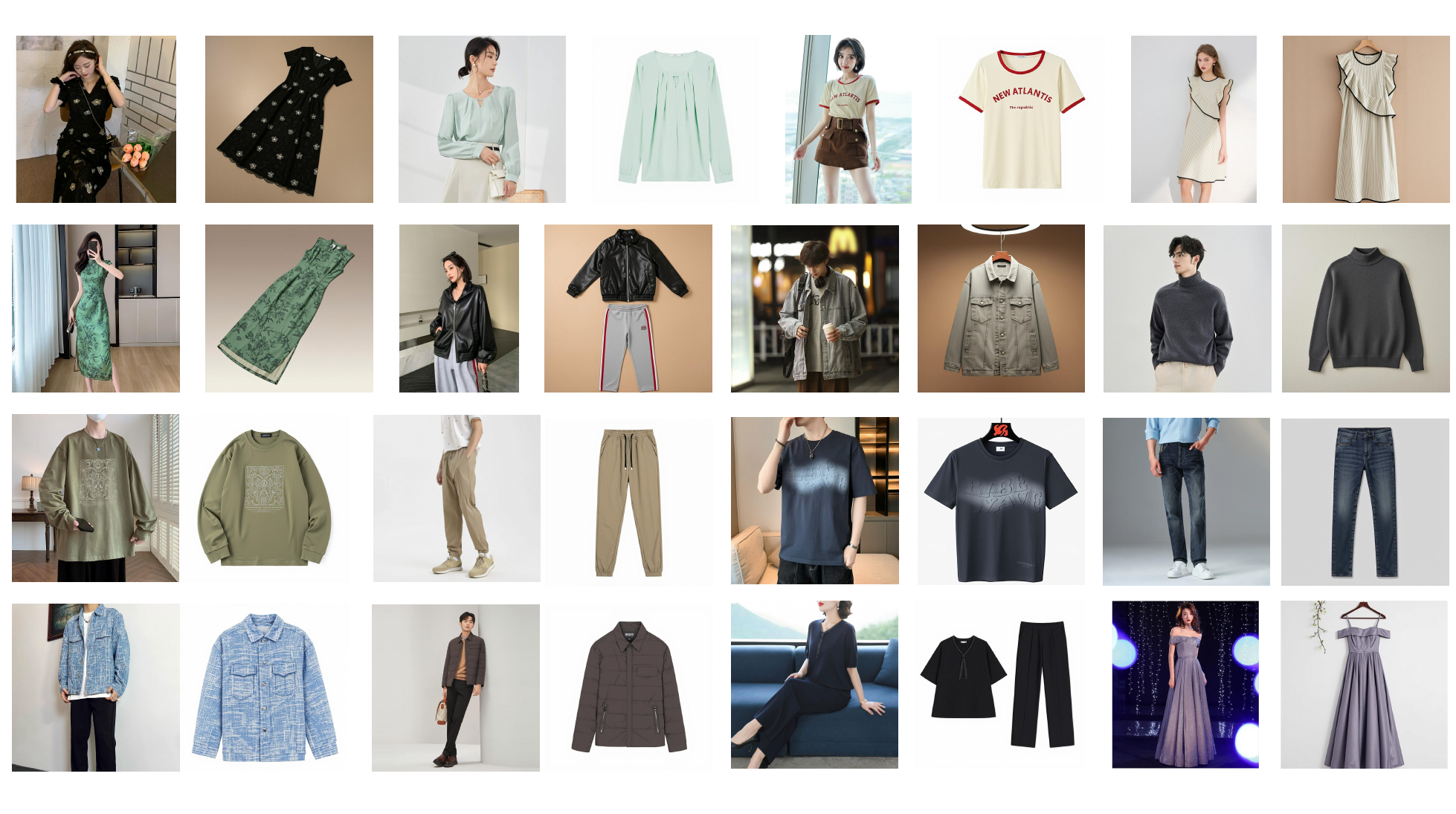}
  \caption{Additional VTOFF showcases in our dataset Omni-TryOn.
  }
  \label{fig:add_takeoff}
\end{figure}

\end{document}